%% file: iclr2024_conference.tex
\documentclass{article} % For LaTeX2e
\usepackage{iclr2024_conference,times}

\usepackage{algpseudocode, algorithm}
\usepackage[noend]{algcompatible}
\input{math_commands.tex}

\usepackage{hyperref}
\usepackage{url}
\usepackage{amsmath, amssymb, amsthm}
\usepackage{bbm, mathtools}
\usepackage{wrapfig}

\newcommand{\bA}{{\mathbf{A}}}

\newcommand{\by}{{\mathbf{y}}}

\newcommand{\bx}{{\mathbf{x}}}
\newcommand{\bz}{{\mathbf{z}}}
\newcommand{\bb}{{\mathbf{b}}}

\newcommand{\bs}{{\mathbf{s}}}

\newcommand{\bu}{{\mathbf{u}}}
\newcommand{\bv}{{\mathbf{v}}}

\newcommand{\btheta}{{\boldsymbol{\theta}}}

\newcommand{\st}{{\rm s.t.}}
\newcommand{\Renyi}{R\'enyi }
\usepackage[utf8]{inputenc} % allow utf-8 input
\usepackage[T1]{fontenc}    % use 8-bit T1 fonts
\usepackage{hyperref}       % hyperlinks
\usepackage{url}            % simple URL typesetting
\usepackage{booktabs}       % professional-quality tables
\usepackage{amsfonts}       % blackboard math symbols
\usepackage{nicefrac}       % compact symbols for 1/2, etc.
\usepackage{microtype}      % microtypography
\usepackage{xcolor}         % colors

\theoremstyle{plain}
\newtheorem{theorem}{Theorem}[section]
\newtheorem{proposition}[theorem]{Proposition}
\newtheorem{lemma}[theorem]{Lemma}

\theoremstyle{definition}

\theoremstyle{remark}
\newtheorem{remark}[theorem]{Remark}
\newcommand*\samethanks[1][\value{footnote}]{\footnotemark[#1]}

\title{$f$-FERM: A  Scalable Framework for  Robust Fair Empirical Risk Minimization}

\author{Sina Baharlouei \thanks{University of Southern California (\texttt{baharlou,razaviya@usc.edu})}
\And Shivam Patel \thanks{Department of Electrical Engineering, IIT Bombay 
(\texttt{shivamapatel2002@gmail.com})}\\
\And Meisam Razaviyayn \samethanks[1]
}

\iclrfinalcopy % Uncomment for camera-ready version, but NOT for submission.
\begin{document}

\maketitle

\vspace{-4mm}
\begin{abstract}
\vspace{-2mm}
Training and deploying machine learning models that meet fairness criteria for protected groups are fundamental in modern artificial intelligence. 
While numerous constraints and regularization terms have been proposed in the literature to promote fairness in machine learning tasks, most of these approaches are not amenable to stochastic optimization due to the complex and nonlinear structure of constraints and regularizers. Here, the term ``stochastic'' refers to the ability of the algorithm to work with small mini-batches of data. Motivated by the limitation of existing literature, this paper presents a unified stochastic optimization framework for fair empirical risk minimization based on $f$-divergence measures ($f$-FERM). The proposed stochastic algorithm enjoys theoretical convergence guarantees. In addition, our experiments demonstrate the superiority of fairness-accuracy tradeoffs offered by $f$-FERM for almost all batch sizes (ranging from full-batch to batch size of one). Moreover, we show that our framework can be extended to the case where there is a distribution shift from training to the test data. 
Our extension is based on a distributionally robust optimization reformulation of $f$-FERM objective under $\ell_p$ norms as uncertainty sets. Again, in this distributionally robust setting, $f$-FERM not only enjoys theoretical convergence guarantees but also outperforms other baselines in the literature in the tasks involving distribution shifts. 
 An efficient stochastic implementation of $f$-FERM is publicly available~\footnote{https://github.com/optimization-for-data-driven-science/f-FERM}.  
\end{abstract}

\vspace{-5mm}
\section{Introduction}
\vspace{-3mm}
Machine learning models are increasingly deployed in critical applications ranging from healthcare~\citep{ahmad2018interpretable} to image processing~\citep{krizhevsky2017imagenet}, education to job recruitment~\citep{boselli2018classifying}, and social networking to cybersecurity~\citep{xin2018machine}. Machine learning practitioners have adopted learning algorithms to fathom inherently difficult and crucial problems. However, na\"ive deployment of these models may lead to serious shortcomings such as biased predictions against minority groups~\citep{angwin2016machine, buolamwini2018gender}, vulnerability to adversarial attacks~\citep{madry2017towards, 7958570, pmlr-v206-baharlouei23a}, or lack of generalizability~\citep{arjovsky2019invariant}. Consequently, it is of utmost importance to have reliable and trustworthy models that are, in particular, fair and comply with equality norms and provisions worldwide~\citep{act1964civil, elford2023equality}.

With the increasing concern for the trustworthiness of unchecked machine learning algorithms, a broad class of paradigms has been proposed to counteract and mitigate both the cause and effects of model unreliability. Imposing statistical independence between model output and particular input features is of interest in various domains, especially when the generalization of a trained model is based on a collection of spurious features present in the training dataset~\citep{dwork2012fairness, hardt2016equality, yan2017learning}. These could be sensitive features like gender, race, age, and/or income in the context of fairness or confounding factors like environmental artifacts in the context of image classification~\citep{arjovsky2019invariant}. Existing literature on imposing statistical independence between selected input features and model outputs is directed into three approaches: pre-processing, post-processing, and in-processing methods. 

Pre-processing methods entail upstream changes made in datasets to mask sensitive features or reduce the dependency of output variables on sensitive features through transforming data in a stage before the training phase~\citep{kamiran2012data, zemel2013learning, ustun2019fairness}. Post-processing methods involve model-specific adjustments to the model's output to ensure the independence of predictions and sensitive attributes~\citep{hardt2016equality, alghamdi2022beyond}. While pre-processing and post-processing methods do not affect the training procedure, they fail to exploit underlying training mechanisms for the best achievable accuracy-fairness tradeoffs. Unsurprisingly enough, optimizing accuracy and fairness jointly (in-processing) leads to better tradeoffs than sequentially optimizing fairness and accuracy in a pre-processing or post-processing fashion. 

In-processing methods alternatively add fairness constraints or regularizers, penalizing dependence between sensitive attributes and output variables.~\citep{zafar2017fairness} utilizes covariance as the measure of independence between the sensitive attributes and the predictions. While such a measure is amenable to stochastic updates, it fails to capture correlations beyond linear. Alternatively, several non-linear measures such as \Renyi correlation~\citep{baharlouei2020renyi}, $\chi^2$ divergence~\citep{lowy2022stochastic}, $L_{\infty}$ distance~\citep{donini2018empirical}, and Maximum Mean Discrepancy (MMD)~\citep{prost2019toward} are proposed in the literature to establish the independence of the predictors and sensitive attributes. In-processing techniques can be model-specific~\citep{DBLP:journals/corr/abs-2111-03015, aghaei2019learning} or generalizable to different training algorithms~\citep{baharlouei2020renyi, lowy2022stochastic}.

In the spirit of in-processing methods, input data-driven constraints or regularization terms are used to modify training objectives of problems like learning generalizable models to new environments, invariant learning, and learning in the presence of distribution shifts~\citep{arjovsky2019invariant, pmlr-v97-mary19a, baharlouei2020renyi}. Such constrained/regularized reformulations are prevalent in learning robust classifiers against adversarial attacks~\citep{sinha2018certifying}, meta-learning~\citep{balaji2018metareg}, federated learning~\citep{deng2023distributed}, and alternative learning paradigms such as learning distributionally robust optimization (DRO) models~\citep{kuhn2019wasserstein, levy2020large}, tilted empirical risk minimization (TERM)~\citep{li2020tilted}, and Squared-root Lasso~\citep{belloni2011square}. 

While in-processing techniques outperform pre-processing and post-processing approaches, they are not scalable to large datasets because of a lack of adaptability to stochastic optimization~\citep{pmlr-v97-mary19a, lowy2022stochastic}. All aforementioned examples consist of regularization terms in their objective functions where the gradient cannot be described as a linear combination of data point functions. As a result, applying stochastic gradient descent or other stochastic first-order methods on the objective functions of such problems might not converge, especially for small batch sizes. 

Motivated by this,~\citep{lowy2022stochastic} proposes a stochastic optimization framework for Exponential \Renyi Mutual Information as the measure of independence.  More recently~\citet{zhong2023learning} use $f$-divergences as regularization terms to establish the independence between sensitive attributes and predictions. They estimate the $f$-divergence regularizers offline through multi-layer neural networks to avoid the computational challenges of devising scalable stochastic methods for nonconvex min-max problems. Our approach, on the other hand, directly solves the variational formulation for both full-batch and stochastic settings with convergence guarantees to non-spurious solutions. In Section~\ref{section: f-divergence}, using the variational representation of $f$-divergences, we present a convergent stochastic optimization framework for fair learning via $f$-divergences.~\citep{lowy2022stochastic} is a special case of $f$-divergences where $f(t) = t^2 -1$ ($\chi^2$ divergence). Aside from $\chi^2$, all other divergences listed in Table~\ref{table: f-divergence-table} are not introduced in the literature to the best of our knowledge.

Designing convergent stochastic algorithms for fair empirical risk minimization can be further explored in scenarios involving changes in the data distribution from the source to the target domain. Detection and mitigation of biases against protected groups in the presence of distribution shifts have been extensively studied in recent years. \citet{lechner2021impossibility} theoretically shows that learning fair representations (pre-processing) is nearly \textit{impossible} for the popular notions of fairness, such as demographic parity in the presence of the distribution shift. \citet{ding2021retiring}, on the other hand, experimentally demonstrates that applying post-processing fairness techniques~\citep{hardt2016equality} to learn fair predictors of income concerning race, gender, and age fails to transfer from one US state (training domain) to another state. Overlooking distribution shifts can lead to catastrophic decisions threatening the well-being of human subjects when deploying a trained model in certain hospitals to other hospitals~\citep{schrouff2022maintaining}. The current literature for handling distribution shifts with in-processing methods relies on certain assumptions on the type of distribution shift (demographic shift~\citep{fang2020rethinking, du2021fair, maity2021does, giguere2021fairness}, label shift~\citep{dai2020label}, and/or covariate shift~\citep{rezaei2021robust, singh2021fairness}) or explicit access to the \textbf{causal graph}~\citep{mishler2022fair, schrouff2022maintaining} of predictors, sensitive attributes, and target variables. As a result, they face practical limitations and cannot cope with most real-world problems involving complex shifts that cannot be categorized in the ones assumed in their works. 

Alternatively,~\cite{taskesen2020distributionally} provides convex objective functions for imposing fairness on logistic regression using constraint optimization. \cite{NEURIPS2019_1770ae9e} use MMD for defining uncertainty sets around training distribution, whereas \cite{NEURIPS2020_8929c70f} use Integral Probability Measure (IPM) to mitigate the distribution shift. The main limitation of these approaches is their reliance on the convexity of the underlying learning model and lack of scalability due to incompatibility with stochastic optimization algorithms.
% \cite{pmlr-v151-terjek22a} extensively compares entropic regularizers for the optimal transport problem with the Sinkhorn-Knopp algorithm for computing the approximate solution.
\citet{wang2023how} uses the Maximum Mean Discrepancy (MMD) distance between the spectral norm of the Hessian matrix at advantaged and disadvantaged data points. However, they do not provide convergence guarantees for their proposed algorithm to any notion of optimality. In addition, the method is not necessarily amenable to stochastic updates. While we naturally define the uncertainty set directly on the joint distribution of sensitive attributes and predictions, they use the curvature of the obtained solution quantified by the norm of the Hessian matrix as a heuristic for promoting the robustness of the fair solution.

\noindent \textbf{Contributions:} This paper establishes a scalable (stochastic) fair empirical risk minimization framework through regularization via $f$-divergences ($f$-FERM) for both standard and distributed shift settings. $f$-FERM presents a unified methodology based on the Legendre-Fenchel transformation, enabling us to develop theoretically convergent first-order stochastic algorithms when only small batches of data are available at each iteration. Further, we have presented the first distributionally robust optimization framework under $\ell_p$ norms uncertainty sets covering nonconvex losses such as neural networks. The presented framework for fair inference in the presence of distribution shifts does not rely on the causal graph describing the causal interaction of input features, sensitive attributes, and target variables, which is rarely available in practical problems.  

\noindent \textbf{Paper Organization:} We structure our response towards designing scalable, robust, and fair algorithms into two sections. Section~\ref{section: f-divergence} motivates the design of unbiased gradient estimators of objectives with information-theoretic $f$-divergence regularizers. In Section~\ref{section: DRO}, we present our approach for fair inference in the presence of the distribution shift in detail. 
% We consider additional cases when prior knowledge of bounds on distributional distance is present. 
Our experiments provide an extensive examination of various $f$-divergences and their suitability as regularizers and also show the consistency of our method across all batch sizes in contrast to existing benchmarks. Similar experiments are carried out for robust training on varying amounts of distributional shifts in data.

\vspace{-4mm}
\section{Fair Empirical Risk Minimization via $f$-divergences}
\vspace{-3mm}
\label{section: f-divergence}
A widely studied problem in algorithmic fairness is promoting a notion of group fairness, such as demographic parity, equalized odds, equality of opportunity, or sufficiency through an in-processing method. For these notions, we aim to establish a [conditional] statistical independence between the predictions (e.g., the creditworthiness of the individual) and the sensitive attributes (e.g., gender, race). For simplicity of presentation, we formulate all problems under the demographic parity notion, which requires statistical independence between the prediction and the sensitive attribute. Without loss of generality, all formulations and methods are generalizable to other aforementioned notions of group fairness by considering conditional random variables (see Appendix~\ref{appendix: notions}). A popular in-processing approach for training fair (classification) models under the demographic parity notion is to regularize the empirical risk minimization:
\vspace{-2mm}
\begin{equation}
\label{eq: ferm}
\vspace{-1mm}
    \min_{\btheta} \quad \frac{1}{n} \sum_{i=1}^n \ell(\hat{y}_{\btheta}(\bx_i), y_i) + \lambda \mathcal{D} \Big(\mathbb{P}({\hat{y}_{\btheta} (\bx), s}), \mathbb{P}(\hat{y}_{\btheta}(\bx)) \otimes \mathbb{P}(s)  \Big),
\end{equation}
where $\btheta$ is the learning parameters (e.g., weights of the neural network); $\bx_i \in \mathbb{R}^d$ is the $i$-th input feature vector; $y_i$ is the actual label/class for sample~$i$; $\hat{y}_{\btheta}(\bx_i)$ is the prediction of the model for sample~$i$; and $\ell(\hat{y}_{\btheta}(\bx_i), y_i)$ is the loss function measuring the ``goodness-of-fit" for sample~$i$. Here,
$\mathcal{D}$ is a divergence between the joint probability distribution of the predictions and sensitive attributes and the Kronecker product of their marginal distributions. Recall that  $\hat{y}_{\btheta}$ and $s$ are statistically independent iff $\mathbb{P}({\hat{y}_{\btheta} (\bx), s})$ follows  $\mathbb{P}(\hat{y}_{\btheta}(\bx)) \otimes \mathbb{P}(s)$. Thus, the second term in~\eqref{eq: ferm} is zero iff $\hat{y}_{\btheta}$ and $s$ are statistically independent (complete fairness under the demographic parity notion). 

% Various divergence measures have been used in the algorithmic fairness literature to establish the independence of sensitive attributes and predictions. \citep{zafar2017fairness} adopts covariance of sensitive attributes and predictions decision boundary as the measure of independence. While this choice of regularizer leads to an almost efficient optimization procedure, unfortunately, it cannot capture non-linear correlations between sensitive attributes and predictions. To address this issue,~\citep{pmlr-v97-mary19a, baharlouei2020renyi} utilize \Renyi correlation as the measure of independence. While these papers derive convergent algorithms for full-batch gradient descent, they fail to guarantee convergence for stochastic gradient descent (small-batch). Therefore, they are not scalable to large datasets. \citep{lowy2022stochastic} addresses this issue by introducing the Exponential \Renyi Mutual Information (ERMI) which is equivalent to $\chi^2$ divergence. It is known that $\chi^2$ divergence falls under the general class of $f$-divergence measures~\citep{DBLP:journals/corr/NielsenN13}. Therefore, a natural question is whether their framework can be naturally generalized to the class of $f$-divergence measures; and if so, whether more efficient fairness-imposing regularizers can be discovered in this class. 

This section studies the fair empirical risk minimization regularized by a broad class of $f$-divergence measures. Let $\mathbb{P}$ and $\mathbb{Q}$ be two discrete probability measures taking values in $\mathcal{P} = \{1, \dots, m\}$. The $f$-divergence between $\mathbb{P}$ and $\mathbb{Q}$ is defined as~\citep[Def 4.9]{polyanskiy2022information}(see Appendix~\ref{appendix: continuous} for the general continuous case):
\vspace{-1mm}
\begin{equation}
\label{eq: f_divergence}
    \mathcal{D}_f(\mathbb{P}, \mathbb{Q}) = \sum_{j = 1}^{m} \mathbb{Q}_j f \Big(\frac{\mathbb{P}_j}{\mathbb{Q}_j}\Big)
\end{equation}
The above definition, which is also known as $f$-mutual information~\citep{lu2023fmicl, csiszar2022information}, covers many known divergence measures used for imposing fairness, such as  KL-divergence for the choice of $f(t)=t\log(t)$~\citep{shui2022learning}, or $\chi^2$ divergence when $f(t) = (t-1)^2$~\citep{lowy2022stochastic}.  As shown in Appendix~\ref{appendix: f-divergence-properties}, $\mathcal{D}_f$  in~\eqref{eq: ferm}  is zero \textit{if and only if} the probability distribution of $s$ and $\hat{y}_{\btheta}$ \textit{are statistically independent} for the choices of $f$ listed in Table~\ref{table: f-divergence-table}. 
In addition, we prove that these $f$-divergences either cover or provide upper bounds for the popular notions of fairness violations in the literature, such as $\ell_p$ distances, \Renyi correlation~\citep{baharlouei2020renyi},  and demographic parity (equalized odds) violation. This means that by minimizing these regularizers, we are minimizing an upper bound of (other) popular fairness violation measures, and thus we are controlling them implicitly. Further, unlike \Renyi correlation~\citep{baharlouei2020renyi, grari2020fairness}, we can utilize Legendre-Fenchel duality (and variational representation) to develop  (provably) convergent algorithms with \textbf{stochastic (mini-batch) updates}. This formulation and the resulting stochastic optimization algorithm are described in the next subsection.

\vspace{-2mm}
\subsection{A Convergent Stochastic Algorithm for fair ERM via $f$-Divergences}
\vspace{-2mm}
Let us start by rewriting~$\eqref{eq: ferm}$ using $f$-divergences as the divergence measure:
\vspace{-2mm}
{\small
\begin{equation}
\label{eq: f-FERM}
\tag{$f$-FERM}
\vspace{-2mm}
    \min_{\btheta} \quad \frac{1}{n} \sum_{i=1}^n \ell(\hat{y}_{\btheta}(\bx_i), y_i) + \lambda \sum_{j \in \mathcal{Y}, \atop k \in \mathcal{S}} \mathbb{P}_{s}(s = k) \mathbb{P}_{\hat{y}_{\btheta}}(\hat{y}_{\btheta}) f \Big( \frac{\mathbb{P}_{\hat{y}_{\btheta}, s}(\hat{y}_{\btheta} = j, s = k)}{\mathbb{P}_{\hat{y}_{\btheta}}(\hat{y}_{\btheta} = j)\mathbb{P}_{s}(s = k)} \Big)  
\end{equation}
}
\normalsize
While the non-linearity of $f$-divergences in~\eqref{eq: f-FERM} empowers the underlying model to capture more complex dependencies between sensitive attributes and predictions compared to the linear measures~\citep{zafar2017fairness}, the objective function can no longer be represented as a summation of functions over input data points. Consequently, one cannot directly apply the stochastic gradient descent method (or its variations, such as Adam) to the objective function in~\eqref{eq: f-FERM}. In particular, directly evaluating the gradient of the objective function of~\eqref{eq: f-FERM} on a mini-batch of data leads to a statistically biased estimation of the entire objective's gradient. Such statistical biases prevent the convergence of algorithms such as SGD (even with a strongly convex minimization landscape)~\citep{DBLP:journals/corr/abs-2008-00051,DBLP:journals/corr/abs-1807-11880}, let aside the more complex objectives arising in modern-day neural networks. 

To derive stochastic algorithms, one can use the variational forms of $f$-divergences  to delineate them as a pointwise supremum of affine transformation over probability densities. The most commonly used and well-behaved transform is the Legendre-Fenchel transform (often called the convex conjugates), which linearizes the dependence of the objective function to input data points using a variational reformulation. Particularly, we can rewrite~\eqref{eq: f-FERM} using the following result:
\begin{proposition}
\label{thm: variational}
Let $f(\cdot)$ be a convex function. Then,
% and $\mathcal{D}_f$ is an $f$-divergence defined as~\eqref{eq: f_divergence}. Therefore,
\eqref{eq: f-FERM} can be reformulated as:
\vspace{-2mm}
\begin{equation}
\label{eq: variational_representation}
\small
\vspace{-2mm}
\min_{\btheta} \: \max_{A} \quad \sum_{i=1}^n \ell(\hat{y}_{\btheta}(\bx_i), y_i) + \lambda \sum_{j \in \mathcal{Y}, \atop k \in \mathcal{S}}  \Big[A_{jk} \mathbb{P}_{\hat{y}, s}(\hat{y}_{\btheta} = j, s = k) -  f^*(A_{jk}) \mathbb{P}_{\hat{y}}(\hat{y}_{\btheta} = j) \mathbb{P}_{s}(s = k) \Big]
\end{equation}
where $f^{*}(z) = \sup_{w \in \textup{dom}(f)} w^T z - f(w)$ is the Legendre-Fenchel  transformation of the function~$f$.
\end{proposition}
\vspace{-4mm}
\begin{proof}
The proof is standard and appears in Appendix~\ref{appendix: thm21}.
\end{proof}
\vspace{-4mm}
In order to solve~\eqref{eq: variational_representation}, we will use (stochastic) first-order methods. Notice that 
% $A_{jk}$ variables are optimization parameters of the maximum problem introduced by the Legendre-Fenchel transformation of $\mathcal{D}_f$. 
% As the model loss function $\ell(\hat{y}_{\btheta}(\bx_i), y_i)$ is independent of $A_k$, the $\text{max}_A$ can be clubbed with the original minimization problem over model parameters $\btheta$ to create a min-max objective where the maximization problem is concave. 
$\mathbb{P}_s(s = k)$ is constant through the optimization procedure and is computed once by counting the number of data points whose sensitive attribute takes the value of $k$: 
% \begin{equation}
% \label{eq: s_estimation}
$
    \pi_k \coloneqq \mathbb{P}_{s}(s = k) = \frac{1}{n} \sum_{i = 1}^{n} \mathbbm{1}(s_i = k).
$
% \end{equation}
Assume we use the softmax layer to compute the probabilities of different classes in our classification task (as it is standard in logistic regression or using neural networks for classification). Let $F_{j} (\bx_i; \btheta)$ be the $j$-th entry of the softmax layer output for datapoint $\bx_i$, predicting the probability of class~$j$.  Then it is easy to show that we can obtain unbiased estimators of $\mathbb{P}_{\hat{y}_{\btheta}}(\hat{y}_{\btheta} = j)$ and   $    \mathbb{P}_{\hat{y}_{\btheta}, s}(\hat{y}_{\btheta} = j, s = k)$ using i.i.d. mini-batch~$\mathcal{B}$ of data points. More precisely, we have
\vspace{-1mm}
{
\small
\begin{equation}\label{eq: unbiasedEstp}
\begin{split} 
    \mathbb{P}_{\hat{y}_{\btheta}}(\hat{y}_{\btheta} = j) & = \frac{1}{n} \sum_{i = 1}^{n} F_{j} (\bx_i; \btheta)= \mathbb{E} \Big[\underbrace{\frac{1}{|\mathcal{B}|} \sum_{i = 1}^{|\mathcal{B}|} F_{j} (\bx_i; \btheta)}_{ \hat{\mathbb{P}}_{\hat{y}_{\btheta}}(j; \: \mathcal{B})} \Big] \\
    \vspace{-1mm}
    \mathbb{P}_{\hat{y}_{\btheta}, s}(\hat{y}_{\btheta} = j, s = k) &= \frac{1}{n} \sum_{i = 1}^{n} F_{j} (\bx_i; \btheta) \mathbbm{1}(s_i = k) = \mathbb{E} \Big[ \underbrace{\frac{1}{|\mathcal{B}|} \sum_{i = 1}^{|\mathcal{B}|} F_{j} (\bx_i; \btheta) \mathbbm{1}(s_i = k)}_{ \hat{\mathbb{P}}_{\hat{y}_{\btheta}, s}(j, k; \: \mathcal{B})} \Big].
\end{split}
\end{equation}
\vspace{-2mm}
}
\normalsize
As a result, Problem~\eqref{eq: variational_representation} can be written as a linearly separable function of input data points ($\bx_i$'s):
\begin{equation}
\label{eq: linearly_separable}    
\small
\min_{\btheta} \: \max_{A} \quad \frac{1}{n}\sum_{i=1}^n \left[ \ell(\hat{y}_{\btheta}(\bx_i), y_i) + \lambda \sum_{j \in \mathcal{Y}, \atop k \in \mathcal{S}}  \Big[A_{jk} F_{j} (\bx_i; \btheta) \mathbbm{1}(s_i = k) -  f^*(A_{jk}) \pi_k F_{j} (\bx_i; \btheta)  \Big] \right]
\end{equation}
Thus, evaluating the gradient of the objective function w.r.t. the variables $\btheta$ and $\bA$ over a random batch of data points leads to an unbiased estimator of the gradient of the objective function. 

In addition to providing an unbiased estimator of gradients, the reformulation~\eqref{eq: linearly_separable} has another crucial property: \textit{the objective function is concave in}~$\bA$. Therefore, optimization problem~\eqref{eq: linearly_separable} falls under the category of nonconvex-concave min-max optimization problems. That is, the objective is (possibly) nonconvex in~$\btheta$ and is concave in~$\bA$.
Thus, we can borrow tools from the (stochastic) nonconvex-concave min-max optimization literature~\citep{pmlr-v119-lin20a,9186144,li2023nonsmooth} to derive a convergent first-order stochastic algorithm as presented in Algorithm~\ref{alg: f-divergence-minibatch}. 
% In other words, we can derive convergent stochastic algorithms by changing the original minimization problem to a min-max optimization and introducing new auxiliary variables $A$. 
We listed the closed-form of $f(\cdot), f^{*}(\cdot),$  for several widely-used $f$-divergence measures including KL-divergence, Reverse KL-divergence, $\chi^2$-divergence, Squared Hellinger distance, Jensen-Shannon divergence, and total variation distance in Table~\ref{table: f-divergence-table}. For the derivation, see Appendix~\ref{appendix: derivations}.

\begin{algorithm}
    \caption{Stochastic Gradient Descent-Ascent (SGDA) for $f$-FERM}
    \label{alg: f-divergence-minibatch}
    \begin{algorithmic}[1]
	 \STATE \textbf{Input}: $\btheta^0 \in \mathbb{R}^{d_{\theta}},$ step-sizes $\eta_\btheta$, $\eta_\alpha$, fairness parameter $\lambda \geq 0,$ iteration number $T$, Batchsize $b$
    \FOR {$t = 1, \ldots, T$}
    \STATE Sample minibatch of data $\mathcal{B}_t = \{(\bx_{t1},\by_{t1}),\cdots,(\bx_{tb},\by_{tb})\}$ 
     % \STATE $dA_{jk} = \nabla_{\bA} (A_{jk}^{t-1} \mathbb{P}_{\hat{\by}_\btheta, \bs} - f^*(A_{jk}^{t-1}) \mathbb{P}_{\hat{\by}_\btheta}\mathbb{P}_{\bs})$
    % \STATE $d \btheta = \nabla_\btheta ( \ell (\btheta^{t-1}, \bx, \by) + \lambda (\bA^{t-1} \mathbb{P}_{\hat{\by}_\btheta, \bs} - f^*(\bA^{t-1})\mathbb{P}_{\hat{\by}_\btheta} \mathbb{P}_\bs)  $
    \STATE $\small \btheta^t = \btheta^{t-1} - \frac{\eta_\btheta}{b} \sum \; \nabla_\btheta \ell(\hat{y}_{\btheta}(\bx), y) - \eta_\btheta \lambda \nabla_\btheta \Big( A_{jk}^{t-1} \hat{\mathbb{P}}_{\hat{\by}_\btheta, \bs}(j, k; \: \mathcal{B}_t) - \pi_k f^*(A_{jk}^{t-1}) \hat{\mathbb{P}}_{\hat{\by}_\btheta}(j; \: \mathcal{B}_t) \Big)$
    
    \STATE $A_{jk}^t = A_{jk}^{t-1} + \eta_\alpha\; \nabla_{\bA} \left( A_{jk}^{t-1} \hat{\mathbb{P}}_{\hat{\by}_\btheta, \bs}(j, k; \: \mathcal{B}_t) - \pi_k f^*(A_{jk}^{t-1}) \hat{\mathbb{P}}_{\hat{\by}_\btheta}(j; \: \mathcal{B}_t) \right)$
    
    \ENDFOR
    \STATE \textbf{Return:} $\btheta^T$
\end{algorithmic}
\vspace{-.03in}
\end{algorithm}
\begin{theorem}
\label{thm: f_divergence_convergence}      (\textbf{Informal Statement}) Assume that $\ell(\cdot, \cdot)$ and $\mathcal{F}_j(\cdot, \btheta)$ are Lipschitz continuous for any given $j$ and $\btheta$ and their gradients are $L$-Lipshitz. Further, assume that $\mathbb{P}(s = k) > 0$ for all protected groups and $\mathbb{P}(\hat{y}_{\btheta} = j) > 0$ at every iteration for all labels $j$. Then, for any given batch size $1 \leq | \mathcal{B}| \leq n$, Algorithm~\ref{alg: f-divergence-minibatch} finds an $\epsilon$-stationary solution of~\eqref{eq: f-FERM} in $\mathcal{O}(\frac{1}{\epsilon^8})$ for any given $\epsilon > 0$. 
\end{theorem}
\vspace{-2mm}
\begin{proof}
The formal statement and proof are relegated to Appendix~\ref{appendix: thm_convergence}.
\end{proof}

Theorem~\ref{thm: f_divergence_convergence} applies to all $f$-divergences listed in Table~\ref{table: f-divergence-table} for all batch-sizes (even as small as the batch size of $1$). 
More sophisticated algorithms can be used to obtain $\mathcal{O}(\epsilon^{-6})$ iteration complexity~\cite{rafique1810non,zhang2022sapd+}. However, such algorithms use nested loops and require more hyperparameter tunings. We provide an example of such an algorithm in Appendix~\ref{appendix: faster_convergence}. If the f-divergence leads to a strongly concave  function in $\mathbf{A}$ or satisfies Polyak-{\L}ojasiewicz condition (e.g., for $\chi^2$ divergence), a faster rate of $\mathcal{O}(\epsilon^{-5})$ can be obtained for this algorithm (Appendix~\ref{appendix: thm_convergence}). In addition, if larger batch size of  $\mathcal{O}(\epsilon^{-2})$ is used, we can further improve this rate to $O(\epsilon^{-4})$ iteration complexity (see Appendix~\ref{appendix: thm_convergence}). Finally, when full batch size is used, then double/triple-loop algorithms can lead to the iteration complexity bounds of $O(\epsilon^{-2})$ in the nonconvex-strongly concave setting and  $O(\epsilon^{-3})$ in the general nonconvex-concave setting; see 
\citep{kong2021accelerated,nouiehed2019solving, ostrovskii2021efficient,thekumparampil2019efficient}.

\begin{table}
\vspace{-10mm}
  \caption{Unbiased Estimators for $f$-divergence Regularizers}
  \label{table: f-divergence-table}
  % \tiny
  \centering
  \begin{tabular}{lll}
    \toprule
    Divergence   & $f(t)$ & The term $r_{jk}$ inside regularizer $\lambda \sum_{j,k} r_{jk}$ in \eqref{eq: linearly_separable} \\
    \midrule
    
    $\chi^2$ & $(t-1)^2$ & $\pi_k [ A_{jk}\mathbb{P}_{\hat{\by}_{\btheta}|\bs_k} - (A_{jk} + \frac{\bA _{jk}^2}{4})\mathbb{P}_{\hat{\by}_{\btheta}} ] $  \\
    Reverse KL & $-\ln t $ & $\pi_k [ A_{jk} \mathbb{P}_{\hat{\by}_{\btheta}|\bs_k} + (1+ \ln (-A_{jk}))\mathbb{P}_{\hat{\by}_{\btheta}} ] $ \\
    Total Variational  & $\frac{1}{2} |t-1|$ & $\pi_k A_{jk} [ \mathbb{P}_{\hat{\by}_{\btheta}|\bs_k}- \mathbb{P}_{\hat{\by}_{\btheta}}] \mathbb{I}_{\{|A_{jk}|<1/2\}}$ \\
    KL  & $t \ln t$ & $ \pi_k [A_{jk} \mathbb{P}_{\hat{\by}_{\btheta}|\bs_k} - e^{A_{jk}-1} \mathbb{P}_{\hat{\by}_\btheta}] $  \\
    Jensen-Shannon  & $ -(t+1) \ln (\frac{t+1}{2}) + t \ln t $ & $ \pi_k [ A_{jk} \mathbb{P}_{\hat{\by}_{\btheta}|\bs_k} + \ln(2 - e^{A_{jk}})\mathbb{P}_{\hat{\by}_\btheta}  ] $ \\
    Squared Hellinger  & $ (\sqrt{t} -1)^2 $ & $ \pi_k [ A_{jk} \mathbb{P}_{\hat{\by}_\btheta|\bs_k} + (A_{jk}^{-1} +2) \mathbb{P}_{\hat{\by}_\btheta}] $\\
   % Cressie-Read ($\alpha$) Divergence & $ \frac{t^{\alpha} - \alpha t - (1-\alpha)}{\alpha (\alpha-1)} $ & \parbox{5cm}{$ \pi_k \Big[A_{jk} \mathbb{P}_{\hat{\by}_\btheta|\bs_k} - \frac{\mathbb{P}_{\hat{\by}_\btheta}}{A_{jk}} \Big( \Big( (A_{jk}-1)\frac{\mathbb{P}_{\hat{\by}_\btheta|\bs_k}}{\pi_k \mathbb{P}_{\hat{\by}_\btheta}} +1 \Big)^{\{ \frac{A_{jk}}{A_{jk}-1} \}} -1\Big) \Big] $}\\
    \bottomrule
  \end{tabular}
\end{table}
\vspace{-4mm}
\section{Robust $f$-FERM in the Presence of Distribution Shifts}
\vspace{-3mm}
\label{section: DRO}
In the previous section, we assumed that the training and test domains have the same distribution. However, this assumption is not necessarily valid in certain applications~\citep{fang2020rethinking}. In particular, a model that behaves fairly on the training data distribution may have an unfair performance in the test phase. To address this issue, this section develops stochastic algorithms for fair empirical risk minimization via $f$-divergences in the presence of the distribution shifts.

Assume that $\hat{\mathbb{P}}_{s, y}(s, \hat{y})$ is the joint distribution of sensitive attributes and predictions on the training data. The distributionally robust fair empirical risk minimization via $f$-divergences is formulated as:
\vspace{-2mm}
\begin{equation}
\label{eq: ferm_dro}
\vspace{-3mm}
    \min_{\btheta} \: \frac{1}{n} \sum_{i=1}^n \ell(\hat{y}_{\btheta}(\bx_i), y_i)  \quad \st \quad \max_{\mathbb{P} \in \mathcal{B}} \mathcal{D}_f \Big(\mathbb{P}({\hat{y}_{\btheta} (\bx), s}) || \:  \mathbb{P}(\hat{y}_{\btheta}(\bx)) \otimes \mathbb{P}(s)  \Big) \leq \kappa. 
\end{equation}
$\mathcal{B} = \mathcal{B}(\hat{\mathbb{P}}, \delta)$ is the distributional uncertainty set defined as a certain ball around the training distribution $\hat{\mathbb{P}}$ with radius $\delta$. This formulation guarantees that the model fairness is preserved  (up to a violence of f-divergence less than $\kappa$) even when the test distribution slightly changes.  
With a slight change of notation, $\hat{\mathbb{P}}$ refers to the training distribution, whereas  $\mathbb{P}$ is the optimization parameter. 

One can define the uncertainty set through an $\epsilon$ neighborhood around the joint distribution of the training data characterized by a distance measure such as $\ell_p$ norms, Wasserstein distance, or MMD distance. While these distributionally robust uncertainty sets are thoroughly analyzed for empirical risk minimization (ERM)~\citep{kuhn2019wasserstein, blanchet2019robust, levy2020large}, the DRO formulation for ERM is limited to the Wasserstein distance for the fair logistic regression~\citep{taskesen2020distributionally} and MMD distance~\citep{wang2023how} on the distribution curvature as a heuristic for robustness. Unfortunately, none of these approaches offer a convergent algorithm with stochastic updates. Further, some of these approaches are limited to special loss functions and heuristics. On the other hand, we study imposing the distributionally robust fairness via $f$-divergences for a general loss function where the uncertainty set is characterized by $\ell_p$ norms (Section~\ref{sec: Lp_norm}) or $f$-divergences (Section~\ref{sec: g_divergence}). Our results show that the former approach is more suitable when lower levels of robustness for fairness are required, and the latter works better for handling larger distribution shifts.
 
\vspace{-2mm}
\subsection{Robust $f$-FERM Under $\ell_p$ Norms  and Small Distribution Shifts}
\vspace{-2mm}
\label{sec: Lp_norm}
This section focuses on the widely studied $\ell_p$ norms as the uncertainty set for the distributional distance between the training and test domains. In this case, Problem~\eqref{eq: ferm_dro} can be written as:
\vspace{-2mm}
\begin{equation}
\label{eq: f-DRO}
\vspace{-1mm}
    \min_\btheta \frac{1}{n} \sum_{i=1}^{n} \ell (\hat{y}_{\btheta}(\bx_i), y_i) \quad \st \quad  \max_{\substack{||\mathbb{P} - \hat{\mathbb{P}}||_p \leq \delta \\ \substack{||\mathbb{Q} - \hat{\mathbb{Q}}||_p \leq \delta}}}  \mathcal{D}_f(\mathbb{P}||\mathbb{Q}) \leq \kappa,
\end{equation}
\vspace{-1mm}
where $\hat{\mathbb{P}}$ represents the joint distribution of the sensitive attributes and predictions and $\hat{\mathbb{Q}}$ denotes the Kronecker product of the marginal distributions between sensitive attributes and predictions.

Since handling non-convex constraints is challenging, as it is standard in training machine learning models, we consider the Lagrangian relaxation of Problem~\eqref{eq: f-DRO} as follows:
\vspace{-2mm}
\begin{equation}
\label{eq: dro_lagrangian}
% \tag{$f$-DRO}
\vspace{-2mm}
\min_\btheta \frac{1}{n} \sum_{i=1}^{n} \ell (\hat{y}_{\btheta}(\bx_i), y_i) + \lambda  \max_{\substack{\|\mathbb{P} - \hat{\mathbb{P}}\|_p \leq \delta \\ \substack{\|\mathbb{Q} - \hat{\mathbb{Q}}\|_p \leq \delta}}}  \mathcal{D}_f(\mathbb{P}||\mathbb{Q})    
\end{equation}
This problem falls under the nonconvex-nonconcave, min-max optimization category and is most likely to be computationally hard for general uncertainty sets~\citep{daskalakis2021complexity}. However, such a min-max optimization problem can be solved to stationarity when the diameter of set $\mathcal{B}$ is small (i.e., under small domain shift), see~\citep{ostrovskii2021nonconvex}. The core idea is to approximate the inner maximization problem with the Taylor approximation, leading to a nonconvex-concave min-max optimization, which is easier to solve~\citep{daskalakis2021complexity, razaviyayn2020nonconvex}. This idea has been used and been successful in machine learning (see~\citet{foret2020sharpness} for its use in Sharpness-aware minimization). Utilizing this idea, Problem~\eqref{eq: dro_lagrangian} can be approximated as:
\vspace{-2mm}
\small
\begin{equation}\label{eq: min-max-small-dro}
% \begin{split}
\vspace{-2mm}
    \min_{\btheta} \max_{\substack{\|\mathbb{U}\|_p \leq \delta \\ \substack{\|\mathbb{V}\|_p \leq \delta}}} 
   \Bigg( h(\btheta, \mathbb{U},\mathbb{V}) := \frac{1}{n} \sum_{i=1}^{n} \ell (\hat{y}_{\btheta}(\bx_i), y_i) 
   % & 
   + \lambda \langle \mathbb{U}, \nabla_{\mathbb{P}} \mathcal{D}_f (\hat{\mathbb{P}} || \hat{\mathbb{Q}}) \rangle 
   % \\
   % & 
   +  \lambda \langle \mathbb{V}, \nabla_{\mathbb{Q}} \mathcal{D}_f (\hat{\mathbb{P}} || \hat{\mathbb{Q}}) \rangle \Bigg), 
% \end{split}
\end{equation}
\normalsize
where we used the change of variables~$\mathbb{U} := \mathbb{P} - \hat{\mathbb{P}}$ and $\mathbb{V} := \mathbb{Q} - \hat{\mathbb{Q}}$.
Equivalently, 
\vspace{-1mm}
\begin{equation}
\label{eq: dro_grad}
\vspace{-1mm}
    \min_{\btheta} \frac{1}{n} \sum_{i=1}^{n} \ell (\hat{y}_{\btheta}(\bx_i), y_i) + \lambda \delta \|\nabla_{\mathbb{P}} \mathcal{D}_f (\hat{\mathbb{P}} || \hat{\mathbb{Q}})\|_q +  \lambda \delta \| \nabla_{\mathbb{Q}} \mathcal{D}_f (\hat{\mathbb{P}} || \hat{\mathbb{Q}})\|_q, 
\end{equation}
where $\|\cdot\|_q$ is the dual of the $\ell_p$ norm with $\frac{1}{p} + \frac{1}{q} = 1$.
\begin{proposition}
\label{thm: epsilon}
Assume that the gradient of the loss function is $L$-Lipshitz, and the second-order derivative of the loss exists. Then, a given $\epsilon-$approximate stationary solution of Problem~\eqref{eq: dro_grad} is an $O(\epsilon)-$approximate stationary solution of Problem~\eqref{eq: dro_lagrangian} whenever $L\delta \lesssim \epsilon  $.
\end{proposition}  

This proposition, which is an immediate application of~\citet[Theorem 3.1]{ostrovskii2021nonconvex}, states that if the desired training accuracy~$\epsilon$ is comparable with the distribution shift amount~$\delta$ (i.e. small distribution shift regime), then one can solve problem~(\ref{eq: dro_grad}) instead of (\ref{eq: dro_lagrangian}). Thus, in this regime, we need to solve~\eqref{eq: dro_grad} or equivalently~\eqref{eq: min-max-small-dro}. To this end, we 
% can compute (sub)gradients~$\nabla_P \mathcal{D}_f (\hat{\mathbb{P}} || \hat{\mathbb{Q}})$ and $\nabla_Q \mathcal{D}_f (\hat{\mathbb{P}} || \hat{\mathbb{Q}})$, and then performing auto-differentiation on~\eqref{eq: dro_grad}. 
need to obtain the  (sub)-gradients of the objective function in~\eqref{eq: min-max-small-dro} w.r.t the $\btheta$, $\mathbb{U}$, and $\mathbb{V}$ variables.  First, notice that % the (sub)gradients w.r.t. $\mathbb{U}$ and $\mathbb{V}$ can be computed by
\[
\nabla_{\mathbb{U}} h(\btheta, \mathbb{U},\mathbb{V}) = \nabla_{\mathbb{P}} D_f (\hat{\mathbb{P}}|| \hat{\mathbb{Q}}) = \boldsymbol{\alpha}^*(\hat{\mathbb{P}},\hat{\mathbb{Q}})
\; \textrm{and} \; 
\nabla_{\mathbb{V}} h(\btheta, \mathbb{U},\mathbb{V}) = \nabla_{\mathbb{Q}} D_f (\hat{\mathbb{P}}|| \hat{\mathbb{Q}}) = f^*(\boldsymbol{\alpha}^*(\hat{\mathbb{P}},\hat{\mathbb{Q}})),
\]
where $\boldsymbol{\alpha}^* (\hat{\mathbb{P}},\hat{\mathbb{Q}}) \in \argmax_{\boldsymbol{\alpha}}  \sum_j \alpha_j \hat{p}_j(\btheta) -  \hat{q}_j(\btheta) f^* (\alpha_j)$. 
Here we invoked Danskin's theorem on the variational form of $D_f$; $\hat{p}_j(\btheta)$ and $\hat{q}_j(\btheta)$ is the $j$-th element of $\hat{\mathbb{P}}$ and  $\hat{\mathbb{Q}}$, respectively.
Next, we need to compute  $\nabla_{\btheta} h(\btheta, \mathbb{U},\mathbb{V})$.  Notice that the derivative of the first term in $h(\cdot)$ w.r.t. $\btheta$ is easy to compute. We next calculate the derivative of the second term of $ h(\btheta, \mathbb{U},\mathbb{V})$ w.r.t. $\btheta$. As the derivative of the third term can be computed similarly, we omit its derivation here.
\begin{equation}
\label{eq: derivative_h_theta}
\begin{split}
\nabla_{\btheta} \langle \mathbb{U}, \nabla_{\mathbb{P}} \mathcal{D}_f (\hat{\mathbb{P}} || \hat{\mathbb{Q}}) \rangle 
% & = \nabla_{\btheta} \langle \mathbb{U}, \nabla_{\mathbb{P}} \sup_{\alpha} ( \sum_j \alpha_j \hat{p}_j(\btheta) -  \hat{q}_j(\btheta) f^* (\alpha_j))\rangle 
% \\
% & 
= \nabla_{\btheta} \langle \mathbb{U} , \boldsymbol{\alpha}^*(\hat{\mathbb{P}},\hat{\mathbb{Q}}) \rangle 
= \sum_j u_j \frac{\hat{q}_j(\btheta) \nabla_{\btheta} \hat{p}_j (\btheta)  - \hat{p}_j(\btheta) \nabla_{\btheta} \hat{q}_j(\btheta) }
{\hat{q}_j^2(\btheta) \times (f^*)^{\prime\prime} (\alpha)|_{\alpha = \alpha_j^* (\hat{\mathbb{P}}, \hat{\mathbb{Q}})}}
\end{split}
\end{equation}
where in the last equation, we used the implicit function theorem to compute the derivative of $\boldsymbol{\alpha}^*$ w.r.t. $\btheta$. Notice that an implicit assumption here is that $f$ is differentiable (which holds for KL-divergence, $\chi^2$ divergence, reverse KL, Jensen-Shannon, and Squared Hellinger distance). Having access to the gradients, we can apply the standard [sub-]gradient descent-ascent algorithm to obtain a solution to Problem~\eqref{eq: dro_grad} (see Appendix~\ref{appendix: dro_alg} for the details).   
\vspace{0.2cm}

\noindent\textbf{A semi-stochastic memory-efficient first-order training algorithm.} To apply (stochastic) gradient descent-ascent  algorithm~\citep{pmlr-v119-lin20a} to problem~\eqref{eq: min-max-small-dro}, we need to have unbiased estimator of the function $h(\btheta, \mathbb{U},\mathbb{V})$ w.r.t. $\btheta$, $\mathbb{U}$, and $\mathbb{V}$ variables. While it seems challenging to obtain unbiased estimator w.r.t. all variables, one can notice that if $\hat{p}_j(\btheta)$ and  $\hat{q}_j(\btheta)$ can be computed easily with one forward pass over all data points (i.e., in $O(m\times n)$ memory requirement). Consequently, the gradient of $h(\btheta, \mathbb{U},\mathbb{V})$ w.r.t. $\mathbb{U}$ and $\mathbb{V}$  can be computed with one forward pass over all data points (without the need for doing backpropagation). On the other hand, one can easily obtain unbiased estimator of $\nabla_{\btheta}\hat{p}_j(\btheta)$ and $\nabla_{\btheta}\hat{q}_j(\btheta)$ in~\eqref{eq: derivative_h_theta} using a small mini-batch of data. 
Such a task requires $O(b \times d)$ memory with $d$ being the number of parameters (i.e., $\btheta \in \mathbb{R}^d$) and $b$ being the batch size. Combining this unbiased estimation with the computed values of $\hat{p}_j(\btheta)$ and  $\hat{q}_j(\btheta)$ leads to an unbiased estimator of the objective of \eqref{eq: min-max-small-dro} w.r.t. $\btheta$ variable. To summarize, we need to do one forward propagation to obtain gradients w.r.t. $\mathbb{U}$ and $\mathbb{V}$, and we only do backpropagation for computing gradients w.r.t. $\btheta$ over the mini-batch of data. Such an algorithm requires $O(mn + bd)$ memory requirement and thus can be used for training large models (with $d,n \gg b,m$). It is known that memory requirements are the major limiting factors in training large models such as LLMs~\citep{malladi2023mezo}.

\vspace{-2mm}
\subsection{Robust $f$-FERM Under $\ell_{\infty}$ Norms and Potentially Large Distribution Shifts}
\vspace{-2mm}
\label{sec: g_divergence}
The developed framework in the previous section assumes the distribution shift is small (the uncertainty set diameter is smaller than a certain threshold). When preserving fairness in the presence of large distribution shifts is a priority, our previous methodology might not work well. 
As discussed before, the formulation \eqref{eq: dro_lagrangian} leads to a nonconvex-nonconcave min-max optimization problem and this class of problems is hard to solve computationally in general (even to stationarity notions). 
Thus, we need to exploit the structure of the problem. In this section, we show that we can exploit the structure to develop a first-order algorithm under large distribution shifts. Particularly, we focus on the case where the uncertainty set is $\ell_{\infty}$ ball and the divergence satisfies certain assumptions (i.e., $f^*(\alpha^*)>0$ and $\alpha^*>0$, which is satisfied for KL divergence). 

Since the function $D_f$ is convex in $\mathbb{P}$ and $\mathbb{Q}$, under $\ell_{\infty}$ uncertainty set on $\mathbb{P}$ and $\mathbb{Q}$, the optimal solution of the maximization problem in~\eqref{eq: dro_lagrangian} will be at an extreme point. Moreover, under the assumption that $f^*(\alpha^*)>0$ and $\alpha^*>0$ (which is satisfied for KL divergence), one can easily see that the optimal $p_j = \min\{\hat{p}_j + \delta,1\}$ and $q_j = \max\{\hat{q_j} - \delta, 0\}$ (see Appendix~\ref{appendix: dro_infinity} for the exact proof). Notice that we need to relax the probability simplex constraint to obtain this efficient, optimal closed-form solution. Thus under this assumption, problem~\eqref{eq: dro_lagrangian} can be reformulated as
\vspace{-1mm}
\begin{equation}
\label{eq: dro_lagrangian_linfty}
\vspace{-1mm}
\min_\btheta \frac{1}{n} \sum_{i=1}^{n} \ell (\hat{y}_{\btheta}(\bx_i), y_i) + \lambda   \mathcal{D}_f(\min\{\mathbb{P} +\delta, 1\} ||\max\{\mathbb{Q} - \delta, 0\}),    
\end{equation}
which is a regular minimization problem and  (sub)gradient descent can be utilized to solve it. 
\vspace{-4mm}
\section{Experiments}
\label{section: experiments}
\vspace{-3mm}
We use three popular notions of group fairness: demographic parity, equalized odds, and equality of opportunity violations (see Appendix~\ref{appendix: notions} for definitions) to measure the fairness of trained models. To run Algorithm~\ref{alg: f-divergence-minibatch}, we set $\eta_{\btheta}$ and $\eta_{\alpha}$ to $10^{-5}$ and $10^{-6}$ respectively in all experiments. Further, by changing $\lambda$, we get different points in the trade-off curve between accuracy and fairness. The range of   $\lambda$ depends on the $f$-divergence (see Appendix~\ref{appendix: hyperparameter} for more information on tuning hyper-parameters). In the inference phase of our experiments, we use the standard maximum likelihood decoding based on the output of the softmax layer, i.e., the predicted label is the label with the highest logit value.

As we will see in this section, several $f$-divergence measures lead to reasonable fairness/accuracy tradeoffs and can outperform existing benchmarks. However,  no single $f$-divergence measure uniformly outperforms other measures in all the experiments. Thus, we believe in applications, the choice of the $f$-divergence can be viewed as a hyperparameter that can be tuned by cross-validation.

\vspace{-2mm}
\subsection{Fairness-Accuracy Tradeoffs on Benchmark Datasets}
\vspace{-2mm}
\begin{wrapfigure}{r}{0.45\textwidth}
\vspace{-8mm}
    \begin{center}    \centerline{\includegraphics[width=0.44\columnwidth]{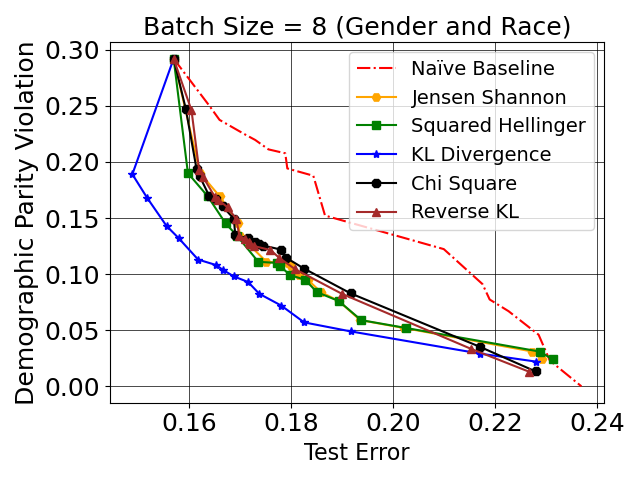}}
    \vspace{-4mm}
    \caption{\small Performance of different $f$-divergences as the regularizers. The experiment is on the adult dataset with gender and race as sensitive attributes. While the offered tradeoffs are close to each other for small demographic parity violations, KL-divergence shows an extraordinary performance for a low-fairness high-accuracy regime. We do not display the performance for larger batch sizes or when only one sensitive attribute is available due to the insignificant difference between the performance of different $f$-divergences.}
    \label{fig: f-divergence}
    \vspace{-4mm}
\end{center}
\end{wrapfigure}
In the first set of experiments, we compare different $f$-divergence formulations for~\eqref{eq: f-FERM} to each other and several state-of-the-art approaches supporting multiple sensitive attributes. Figure~\ref{fig: f-divergence} demonstrates the given tradeoff on the adult dataset~\citep{misc_adult_2} with gender and race as the sensitive attributes (black-female, black-male, white-female, white-male).  To measure fairness, we use the demographic parity violation defined as:
\vspace{-1mm}
\begin{equation*}
\vspace{-1mm}
 \textrm{DPV} = \max_{i, j \in \mathcal{S}} | \mathbb{P}(\hat{y} = 1 | s = i) - \mathbb{P}(\hat{y} = 1 | s = j)|    
\end{equation*}
In the case of binary sensitive attributes (e.g., gender), there is no significant variation between different $f$-divergences. However, when we have $2$ sensitive attributes and the batch size is small ($8$ in Figure~\ref{fig: f-divergence}), the results significantly differ for various $f$-divergences.
Interestingly, KL-divergence for smaller $\lambda$ values shows improvement in fairness violation and accuracy simultaneously.  We do not observe such a phenomenon for other $f$-divergences and state-of-the-art approaches in the literature. Further, in Figure~\ref{fig: adult}, we compare one of the $f$-divergences (reverse KL) to several SOTA methods including~\citet{pmlr-v97-mary19a, baharlouei2020renyi, cho2020fair}. Other approaches such as the pre-processing method of~\citet{zemel2013learning}, post-processing approach of~\citet{hardt2016equality}, and several in-processing methods including~\citet{zafar2017fairness, donini2018empirical, jiang2020wasserstein} demonstrate lower performance compared to the ones depicted in Figure~\ref{fig: adult} and are removed from the figure. While our approach demonstrates consistently good performance across different batch sizes (full-batch, $64$, $8$, $2$), the performances of other methods drop significantly for smaller ones. For further experiments on other datasets (German and COMPAS) and other fairness measures (equality of opportunity and equalized odds violations), see Appendix~\ref{appendix: experiments}.
\begin{figure}[H]
% \vspace{-3mm}
    \begin{center}
\vspace{-3mm}
\centerline{\includegraphics[width=1\columnwidth]{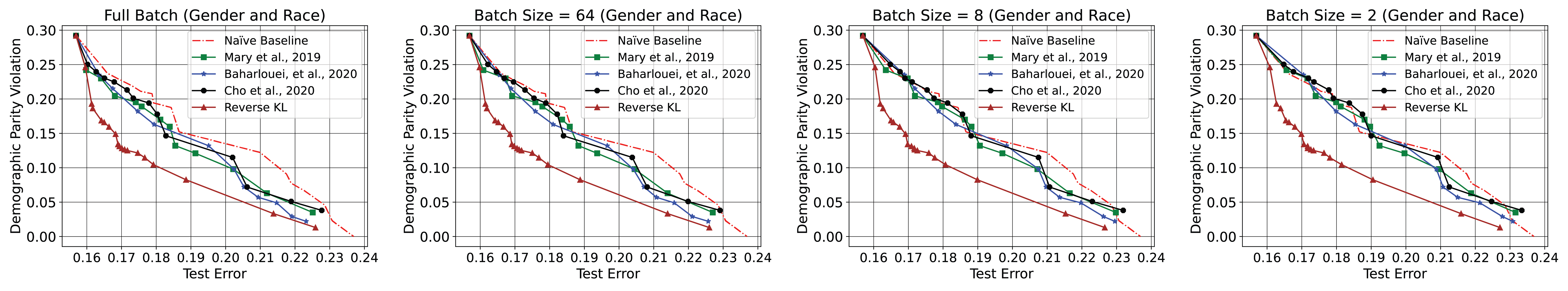}}
    \vspace{-4mm}
    \caption{\small Performance of the trained fair models on Adult Dataset with gender and race as two sensitive attributes with different Batch-sizes. The red dashed line represents the Na\"ive baseline where the model outputs zero with probability $p$. By increasing $p$, the model becomes fairer at the cost of the loss in accuracy.}
    \label{fig: adult}
\end{center}
\vspace{-10mm}
\end{figure}

\subsection{Fairness-Accuracy Tradeoffs in the Presence of the Distribution Shift}
\vspace{-1mm}
We perform two experiments to evaluate the  Algorithms developed in Section~\ref{section: DRO}. In the first experiment, we randomly switch the label of genders for $n\%$  of the data points ($n$ ranges from $1$ to $20$) in the Adult dataset. Then, we train models on the new datasets with a proportion of corrupted sensitive attributes and evaluate the performance on the test data. Figure~\ref{fig: percentage} is obtained by training different models to achieve $80\%$ accuracy on the test data and comparing their demographic parity violation. By increasing the percentage of corrupted sensitive attributes, we see that both $f$-DRO and $f$-infinity achieve less DP violation than SOTA approaches in the literature. In this specific experiment, $f$-DRO works better than $f$-infinity, and there is no significant difference between choosing KL-divergence or $\chi^2$ as the function $f$. Among the papers designed for handling distribution shifts,~\citet{rezaei2021robust} and~\citet{wang2020robust} were the only options with the available implementation.
\begin{wrapfigure}{r}{0.50\textwidth}
% \vspace{-9mm}
\vspace{-3mm}
    \begin{center}    \centerline{\includegraphics[width=0.49\columnwidth]{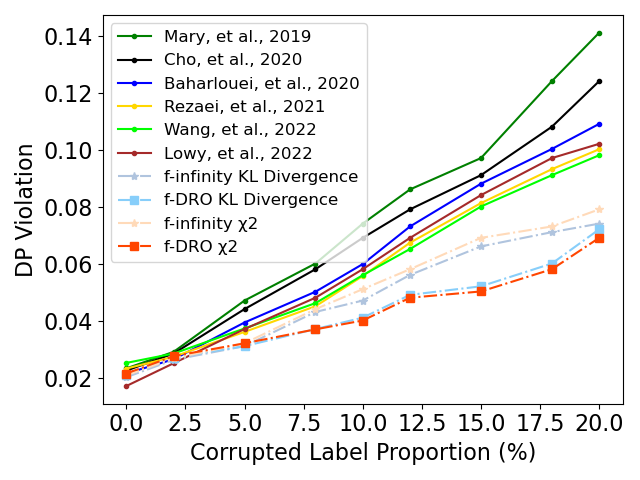}}
    \vspace{-4mm}
    \caption{\small Performance of different state-of-the-art approaches and our two methods for handling distribution shift. The dataset is adult, and the sensitive attribute is gender. We randomly flip the label of a proportion of gender entries (from $0$ to $20\%$). As we observe, our approach demonstrates more robustness against the drop in DP violation compared to other approaches.}
    \label{fig: percentage}
    \vspace{-4mm}
\end{center}
\end{wrapfigure}
In a more recently collected dataset (new adult)~\citep{ding2021retiring}, the users are separated based on their living state. We train different fair models in a single state and evaluate the fairness-accuracy tradeoff in other states. Figure~\ref{fig: new_adult} depicts the performance of different methods. For each method, the center point is the average of accuracy and fairness among $50$ states. The horizontal and vertical lines show the $25$-percentile to $75$-percentile range of performance among the states. The training fairness violation is set to $0.02$ for all methods. We observe that $f$-infinity preserves the fairness level better than other approaches. In comparison, $f$-DRO has a better accuracy. Depending on the application, we suggest using $f$-infinity if preserving a high level of fairness is a priority and $f$-DRO for the cases when a better tradeoff between fairness and accuracy is expected. Note that both approaches offer better fairness-accuracy tradeoffs compared to the SOTA approaches in the literature.  
\begin{figure}[H]
% \vspace{-3mm}
\vspace{-3mm}
    \begin{center}
\centerline{\includegraphics[width=1\columnwidth]{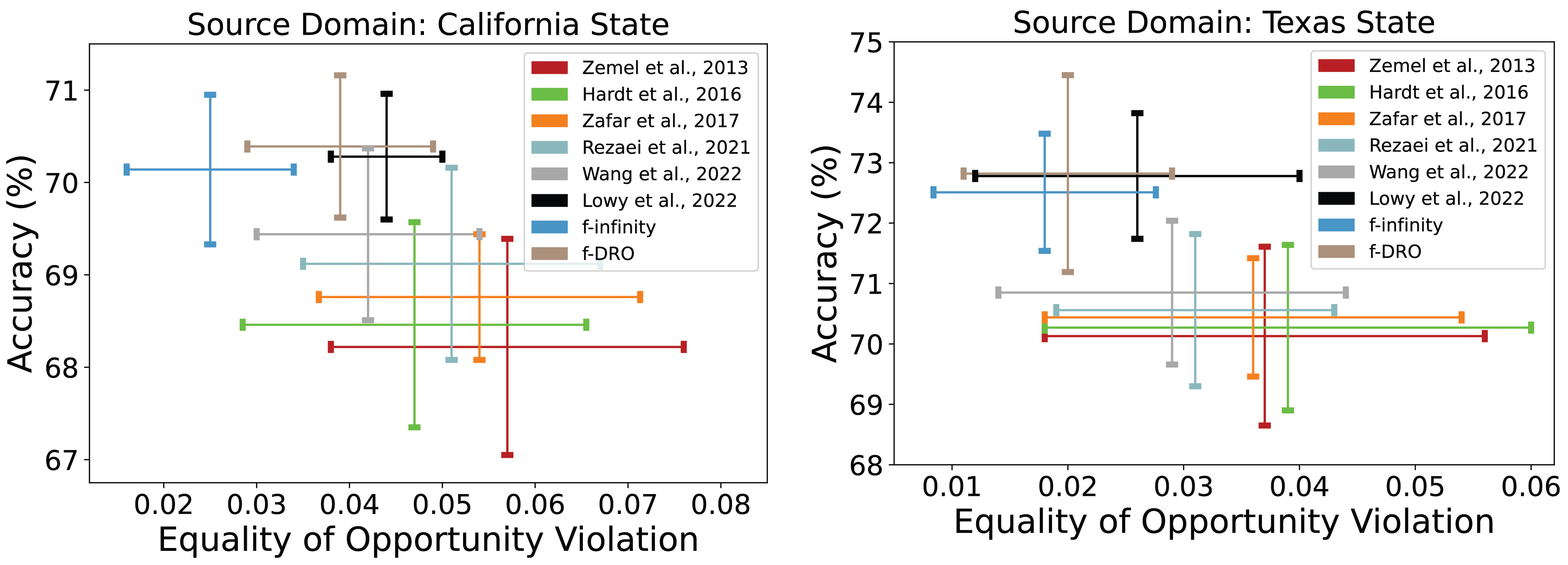}}
    \vspace{-4mm}
    \caption{\small Performance of the trained fair models on new Adult Dataset. The model is trained on one state (California or Texas) and evaluated in $50$ states. The distribution of each state dataset is different than others. Thus, the IID assumption does not hold among datasets of different states.}
    \label{fig: new_adult}
\end{center}
\vspace{-10mm}
\end{figure}

\section{Conclusion}
\vspace{-3mm}
This paper presented a unified stochastic framework for fair empirical risk minimization via $f$-divergences ($f$-FERM). The key idea is to reformulate the objective function as a min-max optimization problem using Legendre-Fenchel duality of $f$-divergence. This enables us to develop an unbiased gradient estimator and a convergent stochastic first-order algorithm. Furthermore, we robustified $f$-FERM using $\ell_p$ norm balls as the uncertainty set against distributional changes. While our empirical investigation delves into the performance and fairness distinctions among various $f$-divergences, a more comprehensive analysis is warranted to determine the optimal $f$-divergence concerning the tradeoff between performance and fairness, faster convergence, and asymptotic behaviors. Furthermore, the distributionally robust formulation of fair empirical risk minimization and the advantages of each formulation can be explored beyond $f$-divergences as the measure of fairness violation and $\ell_p$ norm balls as uncertainty sets. 
\newpage
\section*{Acknowledgements}
This work was supported by the NSF CAREER Award CCF2144985, the AFOSR Young Investigator Program Award FA9550-22-1-0192, a gift from the USC-Meta Center for Research and Education in AI, and a gift from Google.

\bibliography{iclr2024_conference}
\bibliographystyle{iclr2024_conference}

\newpage
\appendix
\section{$f$-FERM for other notions of group fairness}
\label{appendix: notions}
This section shows how we can use alternative notions of fairness, such as equality of opportunity of equalized   odds~\citep{hardt2016equality} instead of demographic parity violation in~\ref{eq: f-FERM}.

Note that a trained model satisfies the equality of opportunity notion for a given binary classifier with a binary sensitive attribute if and only if:
\begin{equation}
\mathbb{P} (\hat{y}_{\btheta}(\bx) = 1, s = i | \:   y = 1) = \mathbb{P} (\hat{y}_{\btheta}(\bx) = 1, s = j | \: y = 1) \quad \forall \: i, j \in \mathcal{S} 
\end{equation}
Therefore, to have a framework for fair inference via $f$-divergences under the equality of opportunity notion, we optimize:
\begin{equation}
\label{eq: f-FERM_eqo}
    \min_{\btheta} \quad \frac{1}{n} \sum_{i=1}^n \ell(\hat{y}_{\btheta}(\bx_i), y_i) + \lambda \mathcal{D}_f \Big(\mathbb{P}(\hat{y}_{\btheta} (\bx), s | y = 1) || \mathbb{P}(\hat{y}_{\btheta}(\bx) | y = 1) \otimes \mathbb{P}(s | y = 1)  \Big).
\end{equation}
Practically, it means that for evaluating the probability measures in the regularization term, we need to focus on the data points whose target labels are $1$.

Further, one can similarly adopt equalized odds as the measure of fairness. Equalized odds as the measure of fairness is defined as:
\begin{equation}
\mathbb{P} (\hat{y}_{\btheta}(\bx) = 1, s = i | \:   y = k) = \mathbb{P} (\hat{y}_{\btheta}(\bx) = 1, s = j | \: y = k) \quad \forall \: i, j \in \mathcal{S}, \: k \in \mathcal{Y} 
\end{equation}
Therefore, we must add a regularizer per each class label to satisfy the equalized odds notion. Other notions of fairness can be used in this framework as long as they can be represented as the conditional independence between sensitive attributes, predictions, and labels~\citep{castelnovo2022clarification}.

\section{$f$-divergences for continuous sensitive attributes and target variables}
\label{appendix: continuous}
In Section~\ref{section: f-divergence}, we developed a framework for promoting fairness for classification problems where both target labels and sensitive attributes are discrete variables. Hence, we could efficiently solve the variational formulation that arose through the designing of unbiased estimators. However, it is not uncommon to find applications of f-divergence regularizers in practice that require either the sensitive features or the output variable to be continuous or both to be continuous parameters. In such cases, the summation over the respective variable is replaced by an integral over the probability distribution. The challenging aspect is calculating the variational form's integral and trailing supremum in the continuous domain.

Let $P$ and $Q$ be two continuous distributions over the space $\Omega$ such that $P$ is absolutely continuous with respect to $Q$ ($P \ll Q$). Then, the $f$-divergence between these two distributions for a given convex function $f$ is defined as:
\begin{equation}
\label{eq: f_divergence_continuous}
    \mathcal{D}_f(\mathbb{P}, \mathbb{Q}) = \int_{\Omega} f \Big(\frac{dP}{dQ}\Big) dQ
\end{equation}
When the target variable is continuous (regression problems), but the sensitive attribute is discrete,~\eqref{eq: f-FERM} can be written as:
\begin{equation*}
    \min_{\btheta} \max_{\bA \in \mathbb{R}^{\infty}} \sum_{i=1}^n \ell(\hat{y}_{\btheta}(\bx_i), y_i) + \lambda \sum_{k} \int_x  \Big[\bA_{k}(x) \mathbb{P}(x) -  f^*(\bA_{k}(x)) \mathbb{Q}_{k}   \Big] dx
\end{equation*}
With slight changes, the above problem can be reformulated as follows:
\begin{equation*}
    \min_{\btheta} \max_{\bA \in \mathbb{R}^{jk}} \sum_{i=1}^n \ell(\hat{y}_{\btheta}(\bx_i), y_i) + \lambda \max_{\bA_1, \dots, \bA_m}\quad\sum_{k} \mathbb{E}\Big[\bA_{k}(s) \mathbb{P}_{j}(s) -  f^*(\bA_{k}(s)) \mathbb{Q}_{k} \Big] 
\end{equation*}
% A compelling idea is to estimate each function $A_{k}(s)$ with a kernel representation. First, we do an empirical copula transformation on the sensitive attribute of each data sample represented by $C_t(s_i)$. Then, we define a kernel transformation matrix as follows: 

% Instead of computing the expectation over a functional space, we maximize a vector $\alpha$ with the dimension equal to the number of elements in the transformation matrix. ({\color{red} should we estimate $P(s)$ as well?})
When both sensitive features and target variables are continuous, the objective function becomes:
\begin{equation*}
    \min_{\btheta} \max_{\bA \in \mathbb{R}^{\infty\times \infty}} \sum_{i=1}^n \ell(\hat{y}_{\btheta}(\bx_i), y_i) + \lambda \int_x \int_y \Big[\bA(x,y) \mathbb{P}(x) -  f^*(\bA(x,y)) \mathbb{Q}(y)   \Big] dx \;dy
\end{equation*}
Such a formulation is clearly intractable for solving $\bA_{k}(x)$ or $\bA(x,y)$ in the continuous domain. We need to approximate the above integrals in discretized/quantized regions or find another variational representation for designing unbiased estimators of continuous domain $f$-divergences. We leave developing algorithms for the continuous target variables and sensitive attributes as a future direction.  

\section{$f$-divergences cover well-known notions of fairness violation}
\label{appendix: f-divergence-properties}
In this section, first, we show that optimizing $f$-divergences to $0$ guarantees the independence of the sensitive attributes and predictions. In other words, optimizing $f$-divergences leads to a fair model under the demographic parity notion (or other group fairness notions discussed in Appendix~\ref{appendix: notions}). 
\begin{proposition} \citep[Theorem 2.3]{polyanskiy2022information}
Let $f$ be a convex function from $\mathbb{R}^+$ to $\mathbb{R}$, such that $f$ is convex, $f(1) = 0$, $f$ is strictly convex in a neighborhood of $1$. Then $\mathcal{D}_{f}(\mathbb{P} || \mathbb{Q}) = 0$, if and only if $P = Q$.
\end{proposition}
As an immediate result, a trained model in~\eqref{eq: f-FERM} is fair under the demographic notion if and only if 
\begin{equation}
\mathbb{P}({\hat{y}_{\btheta} (\bx), s}) = \mathbb{P}(\hat{y}_{\btheta}(\bx)) \otimes \mathbb{P}(s),    
\end{equation}
which means the independence of $s$ and $\hat{y}_{\btheta}(\bx)$.

Next, we show $f$-divergences either include or provide an upper bound for well-known notions of fairness violation in the literature.
\begin{proposition}
\label{proposition: ERMI}
Exponential \Renyi Mutual Information (ERMI)~\citep{pmlr-v97-mary19a, lowy2022stochastic} is an $f$-divergence with $f(t) = (t-1)^2$    
\end{proposition}
\begin{proof}
Exponential \Renyi Mutual Information is defined as~\citep{lowy2022stochastic}:
\begin{equation}
    \textrm{ERMI}(\hat{y}, s) = \sum_{j \in \mathcal{Y}, k \in \mathcal{S}} \frac{\hat{P}_{\hat{y}, s}(j, k)^2}{\hat{P}_{\hat{y}}(j) \hat{P}_{s}(k)} - 1
\end{equation}
For the case of $f(t) = (t-1)^2$, we have:
\begin{gather*}
    \mathcal{D}_f \big(\hat{P}_{\hat{y}} \otimes \hat{P}_{s} || \hat{P}_{\hat{y}, s}\big) = \sum_{j \in \mathcal{Y}}\sum_{k \in \mathcal{S}} \hat{P}_{\hat{y}}(j) \hat{P}_{s}(k) f \Big(\frac{\hat{P}_{\hat{y}, s}(j, k)}{\hat{P}_{\hat{y}}(j) \hat{P}_{s}(k)}\Big)
    = \sum_{j \in \mathcal{Y}}\sum_{k \in \mathcal{S}} \hat{P}_{\hat{y}}(j) \hat{P}_{s}(k) \Big(\frac{\hat{P}_{\hat{y}, s}(j, k)}{\hat{P}_{\hat{y}}(j) \hat{P}_{s}(k)}- 1\Big)^2 \\
    = \sum_{j \in \mathcal{Y}}\sum_{k \in \mathcal{S}} \hat{P}_{\hat{y}}(j) \hat{P}_{s}(k)  \Big(\frac{\hat{P}_{\hat{y}, s}(j, k)^2}{\hat{P}_{\hat{y}}(j)^2 \hat{P}_{s}(k)^2} -2\frac{\hat{P}_{\hat{y}, s}(j, k)}{\hat{P}_{\hat{y}}(j) \hat{P}_{s}(k)}+ 1\Big)  
    \\=\sum_{j \in \mathcal{Y}}\sum_{k \in \mathcal{S}} \Big(\frac{\hat{P}_{\hat{y}, s}(j, k)^2}{\hat{P}_{\hat{y}}(j) \hat{P}_{s}(k)} -2\hat{P}_{\hat{y}, s}(j, k)+ \hat{P}_{\hat{y}}(j) \hat{P}_{s}(k)\Big)\\
    = \sum_{j \in \mathcal{Y}}\sum_{k \in \mathcal{S}} \frac{\hat{P}_{\hat{y}, s}(j, k)^2}{\hat{P}_{\hat{y}}(j) \hat{P}_{s}(k)} - 2 + 1 = \textrm{ERMI}(\hat{y}, s)
\end{gather*}
Note that, in the last equality, we use:
\begin{gather*}
    \sum_{j \in \mathcal{Y}}\sum_{k \in \mathcal{S}} \hat{P}_{\hat{y}, s}(j, k) = \sum_{j \in \mathcal{Y}} \hat{P}_{\hat{y}}(j) = 1,
\end{gather*}
and 
\begin{gather*}
    \sum_{j \in \mathcal{Y}}\sum_{k \in \mathcal{S}} \hat{P}_{\hat{y}}(j) \hat{P}_{s}(k)= \sum_{j \in \mathcal{Y}} \hat{P}_{\hat{y}}(j)\Big(\sum_{k \in \mathcal{S}} \hat{P}_{s}(k)\Big) = \sum_{j \in \mathcal{Y}} \hat{P}_{\hat{y}}(j) = 1,
\end{gather*}
\end{proof}
\begin{proposition}
Demographic parity violation is upper-bounded by the $f$-divergence for $f(t) = (t-1)
^2$\end{proposition}
\begin{proof}
Based on Propositions~\ref{proposition: ERMI}, ERMI is an $f$-divergence with $f(t) = (t-1)^2$. Therefore, the proposition is an immediate result of Lemma 3 in~\citep{lowy2022stochastic}.
\end{proof}
\begin{proposition}
\Renyi correlation~\citep{baharlouei2020renyi} can be upper bounded by the $f$-divergences for the choice of $f(t) = (t-1)^2$.
\end{proposition}
\begin{proof}
Based on Propositions~\ref{proposition: ERMI}, ERMI is an $f$-divergence with $f(t) = (t-1)^2$. Therefore, the proposition is an immediate result of Lemma 2 in~\citep{lowy2022stochastic}.
\end{proof}
\begin{remark}
Mutual Information as the measure of fairness violation~\citep{cho2020fair} is a special case of $f$-divergences for the choice of KL-divergence $f(t) = t\log(t)$ in~\eqref{eq: f-FERM}.    
\end{remark}

\section{Proof of Proposition~\ref{thm: variational}}
\label{appendix: thm21}
\begin{lemma}
\label{thm: master}
Assume that $f(\bz)$ is a semi-continuous convex function. Therefore, $f$ can be written as the following maximization problem:
\begin{equation*}
    f(\bz) = \max_{\alpha} \bz^T \alpha - g(\alpha)
\end{equation*}
where $g$ is the convex conjugate of $f$. 
\end{lemma}

\begin{proof}
Let $g$ be the convex conjugate of the function $f$ defined as:
\begin{equation*}
    g(\alpha) = \sup_{\bz} \: \boldsymbol{\alpha}^T \bz - f(\bz)
\end{equation*}
Since $f$ is a lower semi-continuous convex function, by Fenchel-Moreau theorem~\citep{ioffe2009theory}, it is biconjugate, which means the taking conjugate of $g$ transforms it back to $f$. Therefore, 
\begin{equation*}
    f(\bz) = \sup_{\boldsymbol{\alpha}} \: \boldsymbol{\alpha}^T \bz - g(\boldsymbol{\alpha})  
\end{equation*}
where $g$ is the convex conjugate of $f$.
\end{proof}

Based on the above lemma, we have:
\begin{gather*}
\mathcal{D}_f(\mathbb{P}, \mathbb{Q}) = \sum_{i = 1}^{m} \mathbb{Q}_i f \Big(\frac{\mathbb{P}_i}{\mathbb{Q}_i}\Big) = \mathcal{D}_f(\mathbb{P}, \mathbb{Q}) = \sum_{i = 1}^{m} \mathbb{Q}_i \sup_{\alpha_i \in \textrm{dom} f} \alpha_i \frac{\mathbb{P}_i}{\mathbb{Q}_i} - f^{*}(\alpha_i) \\
= \sup_{\alpha_1, \dots, \alpha_m \in \textrm{dom} f} \sum_{i=1}^m \alpha_i \mathbb{P}_i - f^{*}(\alpha_i) \mathbb{Q}_i 
\end{gather*}
Set $\mathbb{P} = \mathbb{P}({\hat{y}_{\btheta} (\bx), s}), \mathbb{Q} = \mathbb{P}(\hat{y}_{\btheta}(\bx)) \otimes \mathbb{P}(s),$ and  $\alpha_i = A_{jk}$. Therefore, we obtain the formulation in~\eqref{eq: variational_representation}.
\section{Derivation of Closed-Form Expressions for Unbiased Gradient Estimators of $f$-Divergences}
\label{appendix: derivations}

\begin{proposition}
    For two functions $f(t), g(t)$ such that $g(t) = f(t) + c(t-1)$, then $\mathcal{D}_f(\cdot|\cdot) \equiv \mathcal{D}_g (\cdot|\cdot)$. 
\end{proposition}
\begin{proof}
    Proof follows naturally from~\citep[Proposition 7.2]{polyanskiy2022information}
\end{proof}

\begin{theorem}
Let $f(t) = (t-1)^2$ and $\mathbb{P}(s = k) = \pi_k$ ($\chi^2$ Divergence). Then, Equation~\eqref{eq: ferm} can be written as:
\begin{multline}
\label{eq: chi_square}
    \min_{\btheta} \max_{\bA} \sum_{i=1}^n \ell(\hat{y}_{\btheta}(\bx_i), y_i) + \lambda \sum_{j} \sum_{k} \pi_k \Big[A_{jk}\mathbb{P} (\hat{y}_{\btheta} = j | s = k) - (A_{jk} +\frac{A_{jk}^2}{4}) \mathbb{P}(\hat{y}_{\btheta} = j)\Big]
\end{multline}
\end{theorem}
Variational Representation of $f(x)=(x-1)^2$ is given by
\begin{equation*}
    f(x) = \sup_\alpha (\alpha x - f^*(\alpha))
\end{equation*}
Where $f^*(\alpha)$ is the convex conjugate
\begin{equation*}
    f^*(\alpha) = \sup_x(x\alpha - f(x))
\end{equation*}
Taking derivative of $f^*(\alpha)$ w.r.t x gives $x^* = \alpha/2 +1$. This results in $f^*(\alpha) = \alpha + \alpha^2/4$

\begin{theorem}
Let $f(t) = -\ln(t)$ and $\mathbb{P}(s = k) = \pi_k$ (Reverse KL). Then, Equation~\eqref{eq: ferm} can be written as:
\begin{equation}
    \min_{\btheta} \max_{\bA} \sum_{i=1}^n \ell(\hat{y}_{\btheta}(\bx_i), y_i) + \lambda \sum_{j} \sum_{k} \pi_k \Big[A_{jk} \mathbb{P}(\hat{y}_{\btheta} = j | s = k) + (1 + \ln(-A_{jk})) \mathbb{P} (\hat{y}_{\btheta} = j) \Big]
\end{equation}
\end{theorem}
Proceeding as above, optimal $x^*$ for the supremum of $f^*(\alpha)$ is $x^* = -1/\alpha$, \\ resulting in $f^*(\alpha) = -1-\ln(-\alpha)$.

%--------------------
\begin{theorem}
Let $f(t) = \frac{1}{2} |t-1|$  and  $\mathbb{P}(s = k) = \pi_k$ (Total Variational Distance). Then, Equation~\eqref{eq: ferm} can be written as (where $|A_{jk}|\leq \frac{1}{2}$):
\begin{equation}
    \min_{\btheta} \max_{\bA} \sum_{i=1}^n \ell(\hat{y}_{\btheta}(\bx_i), y_i) + \lambda \sum_{j} \sum_{k} \pi_k A_{jk} \Big[ \mathbb{P}(\hat{y}_{\btheta} = j | s=k ) - \mathbb{P}(\hat{y}_{\btheta} = j) \Big]
\end{equation}
\end{theorem}

For $f=\frac{1}{2} |t-1|$, the variational representation is $f(x) = \sup_{\alpha} \Big(\alpha x - f^*(\alpha)\Big)$

Through the convex conjugate $f^*(\alpha)$, we have that 
\begin{align*}
    f^*(\alpha) &= \sup_{x} \Big( x\alpha - f(\alpha) \Big) = \sup_{x} \Big( x\alpha - \frac{1}{2}|x-1| \Big)\\
    &= \begin{cases}
        \infty & \text{for } |\alpha|>\frac{1}{2} \\
        \alpha & \text{for } |\alpha|\leq \frac{1}{2}
    \end{cases}
\end{align*}
So $|\alpha|\leq \frac{1}{2}$ is constrained for the supremum/maximum to exist (otherwise tends to $\infty$).

\begin{theorem}
Let $f(t) = t \ln(t)$  and  $\mathbb{P}(s = k) = \pi_k$ (KL Divergence). Then, Equation~\eqref{eq: ferm} can be written as: 
\begin{equation}
    \min_{\btheta} \max_{\bA} \sum_{i=1}^n \ell(\hat{y}_{\btheta}(\bx_i), y_i) + \lambda \sum_{j} \sum_{k} \pi_k \Big[ A_{jk} \mathbb{P}(\hat{y}_{\btheta} = j|s=k) - e^{A_{jk} - 1} \mathbb{P} (\hat{y}_{\btheta} = j)\Big]
\end{equation}
\end{theorem}

For $f(t)=t \ln(t) $ in f-divergence, the convex conjugate can be represented by:
\begin{align*}
    f^*(\alpha) = \sup_{x} (x\alpha - x \ln(x))
\end{align*}
On differenting w.r.t $x$ for attaining supremum, we get $x = e^{\alpha - 1}$. Hence, the variational representation of $f(t) = t \ln(t)$ becomes:
\begin{align*}
    f(x) = \sup_{\alpha} \Big(x\alpha - e^{\alpha - 1}\Big)
\end{align*}
\textbf{Note:} We can also use the affine transformation $ \alpha \leftarrow \alpha -1$, which results in the more commonly studied version in literature:
\begin{align*}
    D(P||Q) = 1+ \sup_{g:X \rightarrow \mathbb{R}} \mathbb{E}_P[g(X)] - \mathbb{E}_Q[e^{g(X)}]
\end{align*}

\begin{theorem}
Let $f(t) = -(t+1)\ln(\frac{t+1}{2}) + t \ln(t)$  and  $\mathbb{P}(s = k) = \pi_k$ (Jensen-Shannon Divergence). Then, Equation~\eqref{eq: ferm} can be written as: 
\begin{equation}
  \min_{\btheta} \max_{\bA} \sum_{i=1}^n \ell(\hat{y}_{\btheta}(\bx_i), y_i) + \lambda \sum_{j} \sum_{k} \pi_k \Big[ A_{jk} \mathbb{P}(\hat{y}_{\btheta} = j|s=k) + \ln(2 - e^{A_{jk}}) \mathbb{P}(\hat{y}_{\btheta} = j)\Big]
\end{equation}
\end{theorem}
For the JS Divergence, we have $f(t) = -(t+1)\ln(\frac{t+1}{2}) + t \ln(t)$, whose convex conjugate can be represented as:
\begin{align*}
    f^*(\alpha) = \sup_{x} \Big( \alpha x + (x+1)\ln\Big(\frac{x+1}{2}\Big) - x\ln(x) \Big)
\end{align*}
On differentiating w.r.t $x$ to obtain the supremum, we have
\begin{align*}
    \frac{2x}{x+1} = e^{\alpha}
    \implies  x &= \frac{e^{\alpha}}{2-e^{\alpha}}
\end{align*}
Substituting $x$ in $f^*(\alpha)$, 
\begin{align*}
    f^*(\alpha) = -\ln(2-e^{\alpha})
\end{align*}
Thus, in $f(x) = \sup_{\alpha} \Big( x\alpha - f^*(\alpha)\Big)$, we get the variational form as:
\begin{align*}
    f(x) = \sup_{\alpha} \Big( x\alpha + \ln(2-e^{\alpha}) \Big)
\end{align*}

\begin{theorem}
    Let $f(t)$ be \\
\[ 
 f(t) = 
    \begin{cases}
        \frac{t^{\alpha} - \alpha t - (1-\alpha)}{\alpha (\alpha-1)} & \text{if } \alpha \neq 0, \alpha \neq 1 \\
        t \ln(t) -t +1 & \text{if } \alpha = 1\\
        -\ln(t) +t -1 & \text{if } \alpha =0
    \end{cases}
\]
    and  $\mathbb{P}(s = k) = \pi_k$ (General $\alpha$ Divergence). Then, Equation~\eqref{eq: ferm} can be written as: 
    \begin{multline}
  \min_{\btheta} \max_{\bA} \sum_{i=1}^n \ell(\hat{y}_{\btheta}(\bx_i), y_i) + \lambda \sum_{j} \sum_{k} \pi_k \Big[ A_{jk} \mathbb{P}(\hat{y}_{\btheta} = j|s=k)\\
  - \frac{\mathbb{P}(\hat{y}_{\btheta}=j)}{\alpha}\Big( \Big((\alpha -1)A_{jk}+1\Big)^{\frac{\alpha}{\alpha-1}} - 1   \Big)  \Big]
\end{multline}
\end{theorem}
Excluding the limiting cases where $\alpha=1$ or $\alpha=0$, we can find the convex conjugate $f^*(y)$ as:
\begin{align*}
    f^*(y) &= \sup_{x} \Big( xy - f(x)\Big)\\
    &= \sup_{x} \Big(xy - \frac{x^{\alpha} - \alpha x - (1-\alpha)}{\alpha (\alpha-1)}\Big)
\end{align*}
On differentiating w.r.t. $x$, we obtain (here variational parameter is $y$, do not confuse with the constant $\alpha$)
\begin{align*}
    x^* = \Big( (\alpha -1 )y +1\Big)^{\frac{1}{\alpha-1}}
\end{align*}
Thus, 
\begin{align*}
    f^*(y) = \frac{\Big( (\alpha - 1 )y+1\Big)^{\frac{\alpha}{\alpha-1}}}{\alpha} - \frac{1}{\alpha}
\end{align*}
KL Divergence and Reverse KL Divergence can be obtained by taking the limit when $\alpha$ tends to $1$ and $0$, respectively.

\textbf{Note:} Standard literature on divergences often parametrize the $\alpha$-divergence as
\[
f(x) = 
\begin{cases}
    t\ln(t) & \text{if } \alpha = 1\\
    -\ln(t) & \text{if } \alpha = -1\\
    \frac{4}{1-\alpha^2} \Big( 1 - t^{(1+ \alpha/2)} \Big) & \textrm{otherwise}
\end{cases}
\]

This is equivalent to the substitution $\alpha \leftarrow \frac{1+\alpha}{2}$ in the original definition of generalized f-divergence.

\begin{theorem}
    Let $f(t) = (\sqrt{t} - 1)^2$ (equivalently $f(t) = 2(1-\sqrt{t})$) and  $\mathbb{P}(s = k) = \pi_k$ (Squared Hellinger Distance). Then, Equation~\eqref{eq: ferm} can be written as: 
\begin{equation}
  \min_{\btheta} \max_{\bA} \sum_{i=1}^n \ell(\hat{y}_{\btheta}(\bx_i), y_i) + \lambda \sum_{j} \sum_{k} \pi_k \Big[ A_{jk} \mathbb{P}(\hat{y}_{\btheta} = j|s=k) + \mathbb{P}(\hat{y}_{\btheta}=j)\Big(\frac{1}{A_{jk}}+2\Big)  \Big]
\end{equation}
\end{theorem}
For Squared Hellinger Distance, 
\begin{align*}
    f^*(\alpha) &= \sup_x (x\alpha - f(x))\\
    &= \sup_x (x\alpha-2(1-\sqrt x))
\end{align*}
On differentiating w.r.t. $x$, we get
\begin{align*}
    & \alpha+\frac{1}{\sqrt{x}} =0 \; (\text{Note }\alpha<0) \implies x= \frac{1}{(\alpha)^2}\\
    & \implies f^*(\alpha) = \frac{\alpha}{\alpha^2} -2 + \frac{(-2)}{\alpha} = \frac{-1}{\alpha} - 2\\
\end{align*}
Note that the first, second, and third terms are negative, negative, and positive, respectively; hence, the appropriate choice of $\sign(\alpha)$ for functions of odd powers of $\alpha$.

\section{Formal Statement of Theorem~\ref{thm: f_divergence_convergence} and Proof}
\label{appendix: thm_convergence}
\begin{theorem}
\label{thm: formal}
\textbf{Formal Statement of Theorem} Let $(\bx_i, y_i, s_i) \quad \forall 1 \leq i \leq n$ be the collection of $n$ data points satisfying the following assumptions:
\begin{itemize}
    \item $\ell(\cdot, \bx, y)$ is $G$-Lipschitz, and $\beta_\ell$-smooth for all $\bx_i, y_i$. 
    \item $F_{j}(\cdot, \btheta)$ is $L$-Lipschitz and $b$-smooth for all $\btheta$ and all label classes $j$. 
    \item $\widehat{p}_{\hat{y}}^{\min} := \inf_{\{\btheta^t, t \in [T]\}} \min_{j \in [m]} \frac{1}{N} \sum_{i=1}^N \hat{y}_{\btheta, j}(\bx_i) %\mathbbm{1}_{\{ \widehat{y}(\bx_i, \btheta) = j \}}
    \geq \frac{\mu}{2} > 0$.
    %for some $\btheta^* \in \argmin_{\btheta} \max_{W} \widehat{F}(\btheta, W)$. 
    \item $\hat{p}_S^{\min} := \frac{1}{N}\sum_{i=1}^N \mathbbm{1}_{\{s_i = j\}}  > 0$.
\end{itemize}
choose $\eta_{\btheta} = \Theta(\frac{\epsilon^4}{\ell^3 L^2 D^2})$ and $\eta_{\alpha} = \Theta(\frac{\epsilon^2}{\ell\sigma^2})$ and the mini-batch size of $1$. Therefore, Algorithm~\ref{alg: f-divergence-minibatch} finds an $\epsilon$-stationary of Problem~\ref{eq: f-FERM} in $\mathcal{O}(\frac{1}{\epsilon^8})$.
\end{theorem}
\begin{remark}
The \textbf{first assumption} listed in the theorem statement is true for popular losses such as cross-entropy loss and squared loss (assuming that the input data takes values in a bounded set, which holds for all real-world datasets). 
\end{remark}

\begin{remark}
The \textbf{second assumption} holds for popular classifiers generating probability vectors (e.g., logits in neural networks, logistic regression outputs). For classifiers with no probability output, one must transform the output to a number between zero and one first.   
\end{remark}

\begin{remark}
The \textbf{third assumption} states that the probability of assigning a label to the data points must not be zero for all data points for any label at each iteration.
\end{remark}

\begin{remark}
Finally, the \textbf{fourth assumption} ensures each sensitive class's probability is not zero. In other words, there should be at least one point in the dataset with that sensitive attribute for any sensitive group. It holds for all benchmark datasets in practice. Simply put, any protected group appearing during the test phase must have at least one representative in the training data.
\end{remark}

The following lemma is helpful for the proof of the theorem:
\begin{lemma}
\label{lemma: bounded}
    Let $A_1, \dots, A_n$ be $n$ variables such that $\|A_i\|_2 \leq c_i$. Then, we have:
    \begin{equation}
        \mathbb{E}[\|\sum_{i=1}^n A_i\|_2^2] \leq n\sum_{i=1}^n c_i^2
    \end{equation}
\end{lemma}
\begin{proof}
 \begin{gather*}
     \|\sum_{i=1}^n A_i\|_2^2 = \sum \|A_i\|_2^2 + 2 \sum_{i \neq j} \langle A_i, A_j \rangle \leq \sum \|A_i\|_2^2 + \sum_{i \neq j} \|A_i\|_2^2 + \|A_j\|_2^2 = n \sum_{i=1}^n \|A_i\|_2^2,
 \end{gather*}   
 which is based on the fact that $2\langle A_i, A_j \rangle \leq \|A_i\|_2^2 + \|A_j\|_2^2$. Therefore:
 \begin{equation*}
     \mathbb{E}[\|\sum_{i=1}^n A_i\|_2^2] \leq n\sum_{i=1}^n\mathbb{E}[\|A_i\|_2^2] \leq n\sum_{i=1}^n c_i^2
 \end{equation*}
\end{proof}
Now, we are ready to prove Theorem~\ref{thm: formal}.
\begin{proof}
The proof consists of three main steps. First, we need to show that the gradient estimator in Algorithm~\ref{alg: f-divergence-minibatch} is unbiased. Since the samples are IID, for any function $\psi(\cdot, \cdot)$, and an IID batch of data points $\mathcal{B}$ we have:
\begin{equation*}
\mathbb{E} \Big[\frac{1}{\mathcal{B}} \sum_{(\bx, y) \in \mathcal{B}} \nabla \psi(\bx, y) \Big] = \frac{1}{\mathcal{B}} \sum_{(\bx, y)} \mathbb{E}[\psi(\bx, y)] = \mathbb{E}_{(\bx, y) \sim \mathbb{P}(\bx, y, s)}[\nabla \psi(\bx, y)]  
\end{equation*}
As an immediate result, if the objective function is written as the summation over $n$ functions, the gradient estimator over an IID batch of data will be unbiased. According to Equation~\eqref{eq: linearly_separable}, the objective function has the desired form for:
\begin{equation}
\small
\min_{\btheta} \: \max_{\bA} \quad \frac{1}{n}\sum_{i=1}^n \left[ \ell(\hat{y}_{\btheta}(\bx_i), y_i) + \lambda \sum_{j \in \mathcal{Y}, \atop k \in \mathcal{S}}  \Big[A_{jk} F_{j} (\bx_i; \btheta) \mathbbm{1}(s_i = k) -  f^*(A_{jk}) \pi_k F_{j} (\bx_i; \btheta)  \Big] \right]
\end{equation}
Next, we need to show the boundedness of the gradient estimator variance. Let
\begin{equation*}
    G_{\mathcal{B}} =  \frac{1}{|\mathcal{B}|}\sum_{(x_i, y_i) \in \mathcal{B}} \left[ \nabla_{\btheta} \ell(\hat{y}_{\btheta}(\bx_i), y_i) + \lambda \sum_{j \in \mathcal{Y}, \atop k \in \mathcal{S}}  \Big[A_{jk} \nabla_{\btheta} F_{j} (\bx_i; \btheta) \mathbbm{1}(s_i = k) -  f^*(A_{jk}) \pi_k \nabla_{\btheta} F_{j} (\bx_i; \btheta) \right]
\end{equation*}

We need to show for a given data batch:
\begin{equation*}
    \mathbb{E}[\|G_{\mathcal{B}} - G_n\|_2^2] 
\end{equation*}
where $G_n$ is the gradient with respect to all $n$ data points (when $\mathcal{B} = \{1, \dots, n\}$. Note that:
\begin{equation*}
    \|G_{\mathcal{B}} - G_n\|_2^2 \leq 2\|G_{\mathcal{B}}\|_2^2 + \|G_n\|_2^2
\end{equation*}
Thus, it suffices to show that the gradient is bounded for any given $\mathcal{B}$ batch. Since the samples are independent of each other and identically distributed from $\mathbb{P}_{\textrm{train}}$ (IID samples), the second-order moment of the average over $|\mathcal{B}|$ data points is $1/|\mathcal{B}|$ times the variance of a single data point.

Thus, we need to show the boundedness of the gradient for a given data point drawn from the training distribution:
\begin{equation}
\label{eq: grad}
  \left[ \nabla_{\btheta} \ell(\hat{y}_{\btheta}(\bx), y_i) + \lambda \sum_{j \in \mathcal{Y}, \atop k \in \mathcal{S}}  \Big[A_{jk} \nabla_{\btheta} F_{j} (\bx_i; \btheta) \mathbbm{1}(s_i = k) -  f^*(A_{jk}) \pi_k \nabla_{\btheta} F_{j} (\bx_i; \btheta) \right]  
\end{equation}
Based on the first assumption:
\begin{equation}
    \|\nabla_{\btheta} \ell(\hat{y}_{\btheta}(\bx), y_i)\|_2 \leq G
\end{equation}
Based on the second assumption:
\begin{gather}
    \|A_{jk} \nabla_{\btheta} F_{j} (\bx_i; \btheta) \mathbbm{1}(s_i = k)\|_2 \leq LA_{jk} \\
    \| \pi_k f^{*}(A_{jk}) \nabla_{\btheta}F_{j}(\bx_i; \btheta) \|_2 \leq \pi_{k} Lf^*(A_{jk})
\end{gather}
These terms are bounded if $A_{jk}$ is bounded and $f^*(A_{jk})$ is bounded for any $A_{jk}$. This holds true for all $f$-divergences given assumptions 3 and 4. To see why, it suffices to find the optimal solution of each $f$-divergence by setting the gradient zero with respect to $A_{jk}$. In all cases, the solution is a combination of $\mathbb{P}_{s_k}$ and   $\mathbb{P}_{\hat{y}_j}$ terms that are non-zero and bounded (by assumptions 3 and 4). Since each term is bounded in~\eqref{eq: grad}, the expectation of the squared norm is also bounded, according to Lemma~\ref{lemma: bounded}.

Finally, given that the estimator is unbiased, and the variance is bounded (Assumption 4.1 holds in~\citet{pmlr-v119-lin20a}), the two-time-scale stochastic gradient descent-ascent algorithm (which is Algorithm~\ref{alg: f-divergence-minibatch}) finds an $\epsilon$-stationary solution of the Problem in $\mathcal{O}(\frac{1}{\epsilon^8})$ according to Theorem 4.9 in~\citet{pmlr-v119-lin20a}.
\end{proof}

\begin{remark}
 For the case of strongly convex $f$-divergence (e.g. $\chi^2$ divergence), the convergence rate of $\mathcal{O}(\kappa^3 \epsilon^{-4})$ can be obtained (Theorem 4.5 in~\citet{pmlr-v119-lin20a}). Such an iteration complexity holds for the batch size of $\mathcal{O}(\frac{\kappa \sigma^2}{\epsilon^2})$. If the batch size is as small as one, the rate will be $\mathcal{O}(\kappa^3 \epsilon^{-5})$. 
\end{remark}

\begin{remark}
If the batch size is $n$ (deterministic), a rate of $\mathcal{O}(\epsilon^{-6})$ can be obtained according to Theorem 4.8 in~\citet{pmlr-v119-lin20a}. Note that it does not necessarily translate to a better runtime than the stochastic case. Because the per iteration cost of evaluating the gradient for $n$ data points can be much higher than evaluating on just $1$ (or a small number of) data points. 
\end{remark}

\section{A Faster (But Double-Loop) First-order Optimization Algorithm for Optimizing~\eqref{eq: f-FERM}}
\label{appendix: faster_convergence}
We apply SREDA~\citep{luo2020stochastic} to find an $\epsilon$ stationary solution of Problem~\eqref{eq: f-FERM}. Note that SREDA works for non-convex strongly-concave min-max problems. We can directly apply the algorithm when $f$ is the $\chi^2$-divergence. In the cases that the function is concave but not strongly concave (e.g., KL divergence and Reverse KL), we first consider the following approximation:
\begin{equation}
\label{eq: linearly_separable_regularized}    
\small
\min_{\btheta} \: \max_{\bA} \quad \frac{1}{n}\sum_{i=1}^n \left[ \ell(\hat{y}_{\btheta}(\bx_i), y_i) + \lambda \sum_{j \in \mathcal{Y}, \atop k \in \mathcal{S}}  \Big[A_{jk} F_{j} (\bx_i; \btheta) \mathbbm{1}(s_i = k) -  f^*(A_{jk}) \pi_k F_{j} (\bx_i; \btheta) - \frac{\epsilon}{2}\|A_{jk}\|^2 \Big] \right]
\end{equation}
The maximization problem is an $\epsilon$-strongly concave problem. If we apply SREDA (see Algorithm 3 in~\citet{luo2020stochastic}), we find an $\epsilon$ stationary solution of Problem~\eqref{eq: linearly_separable_regularized} in $\mathcal{O}(\kappa^3\epsilon^{-3})$ where $\kappa = \frac{L}{\mu}$ is the condition number. In our case, $\mu$, the strong concavity modulus can be set to the desired accuracy $\epsilon$ so that solving the approximate problem~\eqref{eq: linearly_separable_regularized} leads to an approximate stationary point of the original problem. Therefore, the rate of convergence will be $\mathcal{O}(\epsilon^{-6})$. Note that applying SREDA finds an $\epsilon$ stationary solution of Problem~\eqref{eq: linearly_separable_regularized}. Similar to Theorem~\ref{thm: epsilon}, since the added regularization term is small enough, the obtained solution is a $\mathcal{O}(\epsilon)$-stationary solution of the original problem~\eqref{eq: f-FERM}. An important note is that the SREDA algorithm (line 13 in Algorithm 3~\citep{luo2020stochastic}) has a nested loop compared to the SGDA algorithm proposed in Algorithm~\ref{alg: f-divergence-minibatch}. Therefore, the $\mathcal{O}(\epsilon^{-6})$ iteration complexity bound does not necessarily translate to an improved runtime in practice. Algorithm 2 describes SREDA applied to Problem~\eqref{eq: linearly_separable_regularized}. For the simplicity of the presentation, define the summation argument over $n$ data points as $\phi(\bx_i, y_i, s_i, \btheta, \bA)$.
\begin{algorithm}
    \caption{SREDA Algorithm For Solving~\eqref{eq: f-FERM}}
    \label{alg: SREDA}
    \begin{algorithmic}[1]
	 \STATE \textbf{Input}: periods $q$, m > 0, Number of iterations T,  step-size $\eta_\btheta$, fairness parameter $\lambda \geq 0,$ iteration number $T$, Batchsizes $S$ and $R$. 
    \FOR {$t = 1, \ldots, T$}
    \IF {$t \textrm{ mod } q = 0$}
    
    \STATE Draw $S$ samples $(\bx^{'}_1, y^{'}_1), \dots, (\bx^{'}_S, y^{'}_S)$ 

    \STATE $\bv_t = \frac{1}{S} \sum_{i=1}^S \nabla_{\btheta} \phi(\bx_i, y_i, s_i, \btheta, \bA)$

    \STATE $\bu_t = \frac{1}{S} \sum_{i=1}^S \nabla_{\bA} \phi(\bx_i, y_i, s_i, \btheta, \bA)$

    \ELSE 
    \STATE $\bv_t = \bv^{'}_t$
    \STATE $\bu_t = \bu^{'}_t$
    \ENDIF  
    \STATE $\btheta_{t+1} = \btheta_{t} - \eta_{\btheta} \bv_k$
    \STATE  $(\bA_{t+1}, \bv^{'}_{t+1}, \bu^{'}_{t+1}) = $ ConcaveMaximizer$(t, m, R, \btheta_{t}, \btheta_{t+1}, \bA_t, \bu_t, \bv_t)$
    \ENDFOR
    \vspace{1mm}
    \STATE \textbf{Return:} $\btheta^T$
\end{algorithmic}
\end{algorithm}
The ConcaveMaximizer module is described in Algorithm 4 in~\citet{luo2020stochastic}. %The algorithm convergence rate is $\mathcal{O}(\kappa^3 \epsilon^{-3})$. , which leads to $\mathcal{O}(\epsilon^{-6})$ iteration complexity as $\kappa = \mathcal{O}(\epsilon^{-1})$.

\section{A First-order Optimization Algorithm for Optimizing~\eqref{eq: dro_grad}}
\label{appendix: dro_alg}
This section presents a first-order optimization algorithm for optimizing~\eqref{eq: dro_grad}. The details are presented in Algorithm~\ref{alg: gradient-regularizer}. Further, we show the convergence of the algorithm to an $\epsilon$-stationary solution of Problem~\eqref{eq: dro_grad} in $\mathcal{O}(\epsilon^{-8})$.

\begin{theorem}
    Assume that $\ell(\cdot, \cdot)$ and $\mathcal{F}_j(\cdot, \btheta)$ are Lipschitz continuous for any given $j$ and $\btheta$ and their gradients are $L$-Lipshitz. Further, assume that $\mathbb{P}(s = k) > 0$ for all protected groups and $\mathbb{P}(\hat{y}_{\btheta} = j) > 0$ at every iteration for all labels $j$. Then, for any given batch size $1 \leq | \mathcal{B}| \leq n$, Algorithm~\ref{alg: gradient-regularizer} finds an $\epsilon$-stationary solution of~\eqref{eq: f-FERM} in $\mathcal{O}(\frac{1}{\epsilon^8})$ for any given $\epsilon > 0$. 
\end{theorem}

The proof of the theorem is similar to Theorem~\ref{thm: f_divergence_convergence} as the objective function is non-convex concave. 

One can obtain faster algorithms under additional assumptions. For example, if the set for $\theta$ is assumed to be compact (e.g., we restrict the norm of the weight of the gradient), then we can accelerate the algorithm to $O(\epsilon^{-6})$, see~\citet{rafique2022weakly}. Moreover, if we consider full batch sizes, we can utilize Algorithm 2 in~\citet{ostrovskii2021nonconvex}. This will give you the rate of convergence of $O(\epsilon^{-2})$ (Theorem 5.2).

\begin{algorithm}
    \caption{Gradient-Regularization Robust Training algorithm}
    \label{alg: gradient-regularizer}
    \begin{algorithmic}[1]
	 \STATE \textbf{Input}: $\btheta^0 \in \mathbb{R}^{d_{\theta}},$ step-sizes $\eta_\btheta$, $\eta_\alpha$, fairness parameter $\lambda \geq 0,$ iteration number $T$, Batchsize $[b]_t$
    \FOR {$t = 1, \ldots, T$}
    \STATE Sample minibatch of data $\bb_t = \{(\bx_1,\by_1),\cdots,(\bx_b,\by_b)\}$ 
    \STATE Estimate $\mathbb{P}(\hat{\by}_{\btheta^{t}})$ for minibatch $\bb_t$
    \REPEAT
        \STATE $dA_{jk} = \nabla_{\bA} (A_{jk} \mathbb{P}_{\hat{\by}_\btheta, \bs} - f^*(A_{jk}) \mathbb{P}_{\hat{\by}_\btheta}\mathbb{P}_{\bs})$
        \STATE $A_{jk} = A_{jk} + \eta_\alpha\; dA_{jk}$
    \UNTIL {Convergence to $A_{jk}^*$}
    \STATE Obtain closed form expressions: $\frac{\partial}{\partial \btheta} ||\nabla_\mathbb{P} \mathcal{D}_f(\mathbb{P}||\mathbb{Q})||_2^2$ and $\frac{\partial}{\partial \btheta} ||\nabla_\mathbb{P} \mathcal{D}_f(\mathbb{P}||\mathbb{Q})||_2^2$ in terms of $\mathbb{P}_{\hat{\by}_\btheta}$ 
    
    \STATE 
        $d \btheta = \nabla_\btheta \Big[ \ell (\btheta^{t-1}, \bx, \by) + \lambda\Big[  \mathcal{D}_f(\hat{\mathbb{P}}||\hat{\mathbb{Q}}) + \epsilon \Big(|| \nabla_\mathbb{P} \mathcal{D}_f(\hat{\mathbb{P}} || \hat{\mathbb{Q}}) ||^2_2 + || \nabla_\mathbb{Q} \mathcal{D}_f(\hat{\mathbb{P}} || \hat{\mathbb{Q}}) ||^2_2 \Big)\Big] \Big] $ 
    \STATE $\btheta^t = \btheta^{t-1} - \eta_\btheta\; d\btheta $
    \ENDFOR
    \STATE \textbf{Return:} $\btheta^T$
\end{algorithmic}
\vspace{-.03in}
\end{algorithm}

\section{Proof of Equation~\eqref{eq: dro_lagrangian_linfty}}
\label{appendix: dro_infinity}
\vspace{-2mm}
To show Problem~\eqref{eq: dro_lagrangian} is equivalent to~\eqref{eq: dro_lagrangian_linfty} under $\ell_p$ norm and the probability simplex constraint relaxation, note that the maximization problem in~\eqref{eq: dro_lagrangian} is a constrained convex maximization with respect to $\mathbb{P}$. Therefore, there is a global solution on the boundary. The maximum problem under $\ell_{\infty}$ can be written as:
\begin{equation}
\max_{\substack{||\mathbb{P} - \hat{\mathbb{P}}||_{\infty} \leq \delta \\ \substack{||\mathbb{Q} - \hat{\mathbb{Q}}||_{\infty} \leq \delta}}}  \mathcal{D}_f(\mathbb{P}||\mathbb{Q}), 
\end{equation}
or equivalently:
\begin{equation}
\max_{\substack{||\mathbb{P} - \hat{\mathbb{P}}||_{\infty} \leq \delta \\ \substack{||\mathbb{Q} - \hat{\mathbb{Q}}||_{\infty} \leq \delta}}}  \sum_{j=1}^{m}\mathbb{Q}_j f \Big(\frac{\mathbb{P}_j}{\mathbb{Q}_j} \Big), 
\end{equation}
For the choice of KL-divergence ($f(t) = t\log(t)$) and $\chi^2$ divergence ($f(t) = (t-1)^2$), $f$ is a non-decreasing function. Fixing a $j \in \{1, \dots, m\}$, the maximum with respect to is $\mathbb{P}_j$  attained when $\mathbb{P}_j$ is maximized. The maximum of $\mathbb{P}_j$ is obtained on the boundary where $\delta$ is added to $\hat{\mathbb{P}}_j$. Since $\hat{\mathbb{P}}_j + \delta$ should be a probability value, if it is larger than $1$, we project it back to $1$. As a result, the maximum of $\mathbb{P}_j$ is $\max(\hat{\mathbb{P}}_j + \delta, 1)$. Further, $f$ in both choices of $f$ is super-linear, meaning that $\mathbb{Q}_j f \Big(\frac{\mathbb{P}_j}{\mathbb{Q}_j} \Big)$ is a non-increasing function with respect to $\mathbb{Q}_j$. Thus, its maximum with respect to $\mathbb{Q}_j$ is attained when $\mathbb{Q}_j$ is minimized. Therefore, the optimal solution is either $\hat{\mathbb{Q}}_j - \delta$, or if it goes less than $0$, we project it back to $0$. Applying the same argument to all $j$'s, we obtain Equation~\eqref{eq: dro_lagrangian_linfty}.

\section{Details of Tuning Hyperparameters}
\label{appendix: hyperparameter}
\vspace{-2mm}
In all experiments, we set $\eta_{\btheta} = 10^{-5}$ and $\eta_{\alpha} = 10^{-6}$. Further, we train the model with $\lambda = 0$ for $300$ epochs, and then we set $\lambda$ to the considered value. We continue the training until $2000$ epochs. The range of $\lambda$ to get each point in the tradeoff figures is varied for different $f$-divergences. The KL-divergence $\lambda$ range is $[0, 150]$. For $\chi^2$ divergence it is $[0, 300]$ and for the reverse KL it is $[0, 50]$. Moreover, the $\lambda$ range for JS and Squared Hellinger is $[0, 110]$ and $[0, 250]$. Note that larger values outside the range lead to models with $0$ predictions for all values.

In the DRO case, aside from $\lambda$ we must tune $\epsilon$, the robustness parameter. To achieve the best result, we have two different strategies depending on the availability of the data from the target domain. Suppose we have access to a collection of data points from the target domain. In that case, we consider it as the validation set to choose the optimal combination of $\lambda \in \{0.1, 0.5, 1, 2, 5, 10, 20, 50\}$ and $\delta \in \{0.01, 0.02, 0.05, 0.1, 0.2, 0.5, 1, 2, 5, 10\}$. In the second scenario, when we do not have any access to target domain data, we perform a $k$-fold cross-validation on the source data. A more elegant way is to create the validation dataset by oversampling the minority groups. Having access to the oversampled validation set, we choose the optimal $\lambda$ and $\delta$ similar to the first scenario. In the experiment regarding Figure~\ref{fig: new_adult}, we reserve $5\%$ of data from the target domain for validation (scenario 1). In Figure~\ref{fig: adult}, we apply scenario 2 to tune the hyperparameters $\lambda$ and $\delta$.

\section{Further Experiments on Other Datasets and Notions of Fairness}
\label{appendix: experiments}
In this section, we perform~\eqref{eq: f-FERM},~\citep{hardt2016equality},~\citep{pmlr-v97-mary19a}, and~\citep{baharlouei2020renyi} to COMPAS~\footnote{https://www.kaggle.com/datasets/danofer/compass} and German Credit datasets~\footnote{https://archive.ics.uci.edu/dataset/144/statlog+german+credit+data}. In the experiment on COMPAS, we use equality of opportunity as the measure of fairness violation, while in the German Credit dataset experiment, we use equalized odds. The results show that~\eqref{eq: f-FERM} is significantly better than other approaches regarding the accuracy-fairness tradeoff. The batch size is equal to $64$ for all methods.
\begin{figure}[H]
% \vspace{-3mm}
    \begin{center}
\centerline{\includegraphics[width=1\columnwidth]{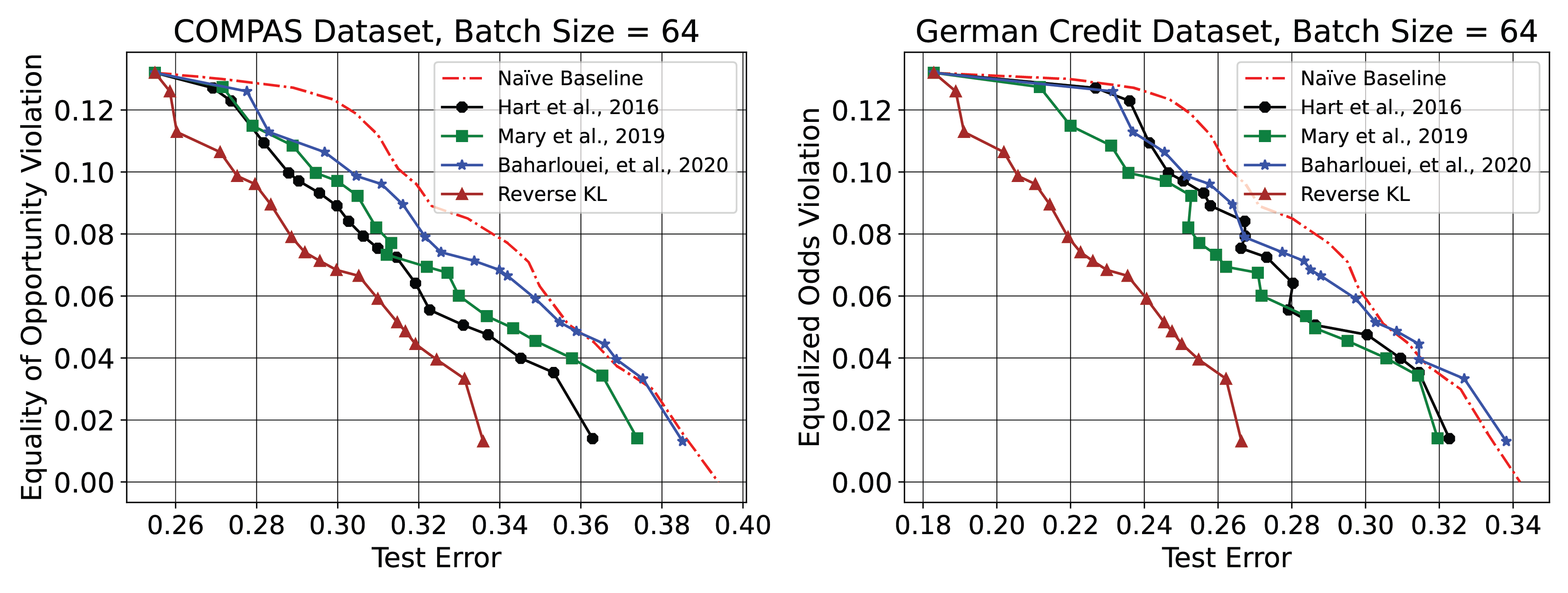}}
    \vspace{-3mm}
    \caption{\small Performance of the trained fair models on COMPAS and German Credit Datasets.}
    \label{fig: compas}
\end{center}
\vspace{-10mm}
\end{figure}
\end{document}

%% file: math_commands.tex
\usepackage{amsmath,amsfonts,bm}

\def\1{\bm{1}}

\DeclareMathAlphabet{\mathsfit}{\encodingdefault}{\sfdefault}{m}{sl}
\SetMathAlphabet{\mathsfit}{bold}{\encodingdefault}{\sfdefault}{bx}{n}

\DeclareMathOperator*{\argmax}{arg\,max}

\DeclareMathOperator{\sign}{\textrm{sign}}